\documentclass{article}

    \PassOptionsToPackage{numbers, compress}{natbib}

 \usepackage[main, final]{neurips_2026}

\usepackage[utf8]{inputenc} 
\usepackage[T1]{fontenc}    
\usepackage{hyperref}       
\usepackage{url}            
\usepackage{booktabs}       
\usepackage{amsfonts}       
\usepackage{nicefrac}       
\usepackage{microtype}      
\usepackage{xcolor}         
\usepackage{microtype}
\usepackage{booktabs, multirow,threeparttable}
\usepackage{graphicx}
\usepackage{placeins}   
\usepackage{tabularx}
\usepackage{array}
\usepackage{colortbl}
\usepackage{xcolor}
\usepackage{makecell}
\usepackage[normalem]{ulem}
\usepackage{subcaption}
\usepackage{adjustbox}
\usepackage{booktabs} 
\usepackage{xcolor}
\usepackage{wrapfig}
\newcommand{\method}{\texttt{CORAL}}
\usepackage{subcaption}
\usepackage{colortbl}
\usepackage[table]{xcolor}

\usepackage{amsmath}
\usepackage{amssymb}
\usepackage{mathtools}
\usepackage{amsthm}
\usepackage{enumitem}

\theoremstyle{plain}
\newtheorem{theorem}{Theorem}[section]

\newtheorem{lemma}[theorem]{Lemma}

\theoremstyle{definition}

\theoremstyle{remark}

\title{Mitigating Object Hallucination in Large Vision-Language Models via False Discovery Controlled Visual Data Splitting}

\author{%
\begin{tabular}{@{}
  >{\centering\arraybackslash}p{0.31\textwidth}
  @{\hspace{0.02\textwidth}}
  >{\centering\arraybackslash}p{0.31\textwidth}
  @{\hspace{0.02\textwidth}}
  >{\centering\arraybackslash}p{0.31\textwidth}
@{}}
 Chang Liu
&
 Yu Tian
&
 Rui Xie\thanks{Corresponding author.}
\\
\normalfont School of Data, Mathematical, and Statistical Sciences
&
\normalfont Institute of Artificial Intelligence
&
\normalfont School of Data, Mathematical, and Statistical Sciences
\\
\normalfont University of Central Florida
&
\normalfont University of Central Florida
&
\normalfont University of Central Florida
\\
\texttt{chang.liu@ucf.edu}
&
\texttt{yu.tian2@ucf.edu}
&
\texttt{rui.xie@ucf.edu}
\end{tabular}%
}

\begin{document}

\maketitle

\begin{abstract}
Multiple object hallucination, where large vision-language models (LVLMs) generate objects not supported by the visual input, is a persistent challenge caused by visual uncertainty during decoding. 
Existing methods reduce hallucinations using contrastive signals, but they rely on heuristics and lack principled control of false positives at the image level. 
To address this, we propose False Discovery Rate-\textbf{CO}nt\textbf{R}ol of H\textbf{AL}lucination ($\method$), a training-free framework that models visual uncertainty using an uncertainty-aware visual {data splitting} strategy and leverages {mirror statistic} to quantify visual contrast during decoding. By computing mirror statistic from paired, symmetrically perturbed visual inputs, $\method$ estimates spurious object predictions and sets a data-driven threshold to control the expected fraction of false discoveries per image, suppressing hallucinations while retaining high power for truly grounded objects. 
The framework is flexible, supports multiple LVLMs, and mitigates hallucinations without retraining or supervision. Extensive experiments on multiple benchmarks with several evaluation metrics demonstrate that $\method$ consistently outperforms state-of-the-art methods, providing more reliable and robust hallucination control.
Code is available at: \url{https://changliu1993-cl.github.io/CORAL/}
\end{abstract}

\section{Introduction}
 \begin{wrapfigure}{r}{0.5\linewidth}
    \centering
    \vspace{-0.17in}
    \includegraphics[width=\linewidth]{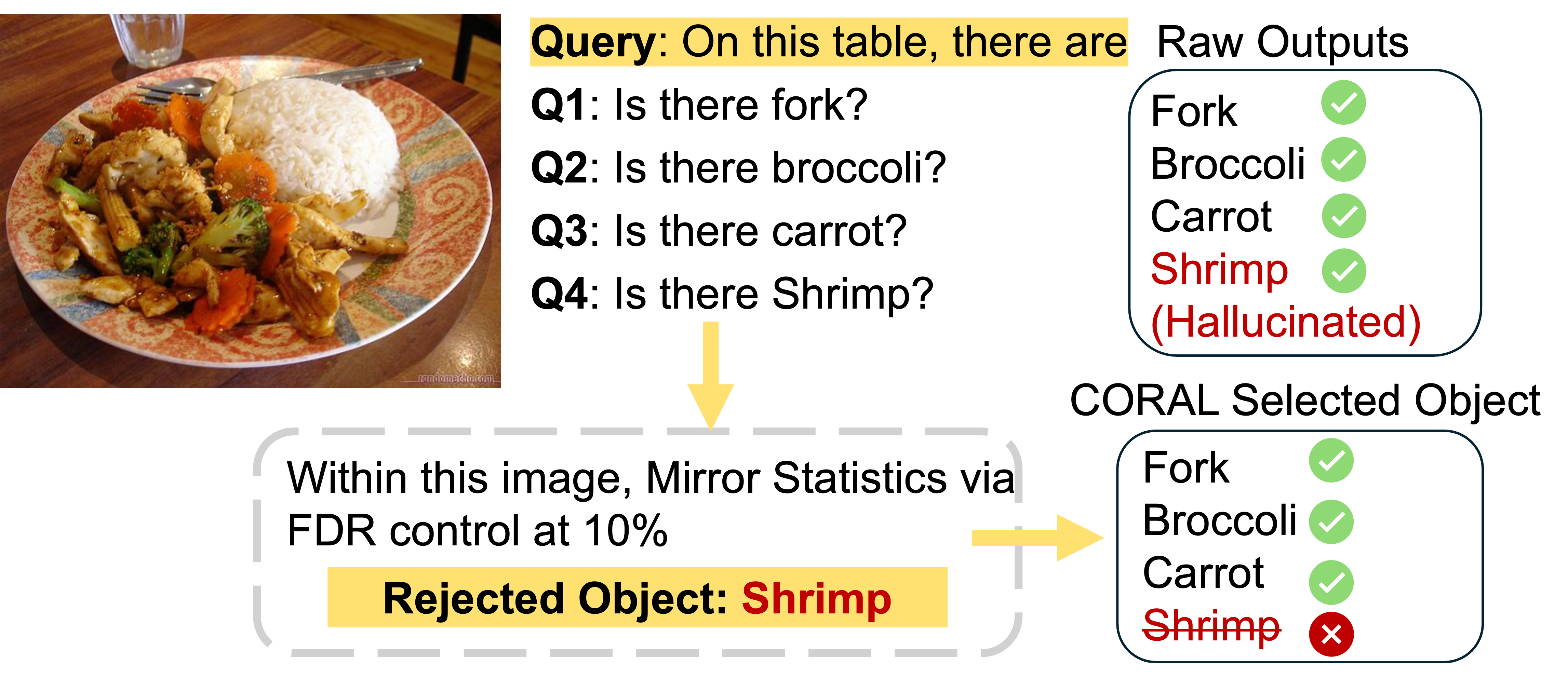}
    \caption{\textbf{Hallucination under multiple queries.} The model may generate plausible but hallucinated objects (e.g., \textit{``shrimp''}, in \textcolor{red}{red}). $\method$ performs image-level selection, rejecting hallucinations while retaining visually grounded objects under False Discovery Rate (FDR) control.}
    \label{fig:multi_object}
    \vspace{-0.15in}
 \end{wrapfigure}

Large Vision-Language Models (LVLMs) extend large language models with pretrained vision encoders to jointly reason over visual and textual inputs, enabling the generation of linguistically fluent and semantically aligned outputs grounded in visual content~\citep{liu2023visual,li2025benchmark}. This unified multimodal framework has led to strong performance across core vision–language tasks such as image captioning~\citep{li2023blip}, visual question answering~\citep{rohrbach2018object}, and multimodal dialogue~\citep{chen2023shikra}, further driven by advances in model architectures~\citep{zhu2023minigpt,ye2023mplug,gao2023llama}, multimodal alignment techniques~\citep{yu2024rlhf,deng2024enhancing}, and large-scale benchmarks~\citep{lu2023mathvista,xu2024lvlm}. Despite this progress, LVLMs remain vulnerable to hallucinations, where models produce confident yet visually ungrounded descriptions of objects, attributes, or relationships absent from the input image~\citep{liu2023visual,zhu2023minigpt,ye2023mplug,gao2023llama,huang2025survey}. Such hallucinations are pervasive across tasks, including image captioning~\citep{lin2014microsoft} and visual question answering~\citep{schwenk2022okvqa}, and persist under both closed-set and open-ended evaluation protocols~\citep{li2023evaluating,wang2023evaluation,zhou2023analyzing}.
{As shown in Fig.~\ref{fig:multi_object}, LVLMs may produce plausible yet hallucinated objects when answering multiple queries about the same image. Evaluating each query independently is insufficient, as hallucinated objects (e.g., \textit{shrimp}) can be retained alongside real ones. }

\begin{figure*}[htbp]
    \centering
    \includegraphics[width=\linewidth]{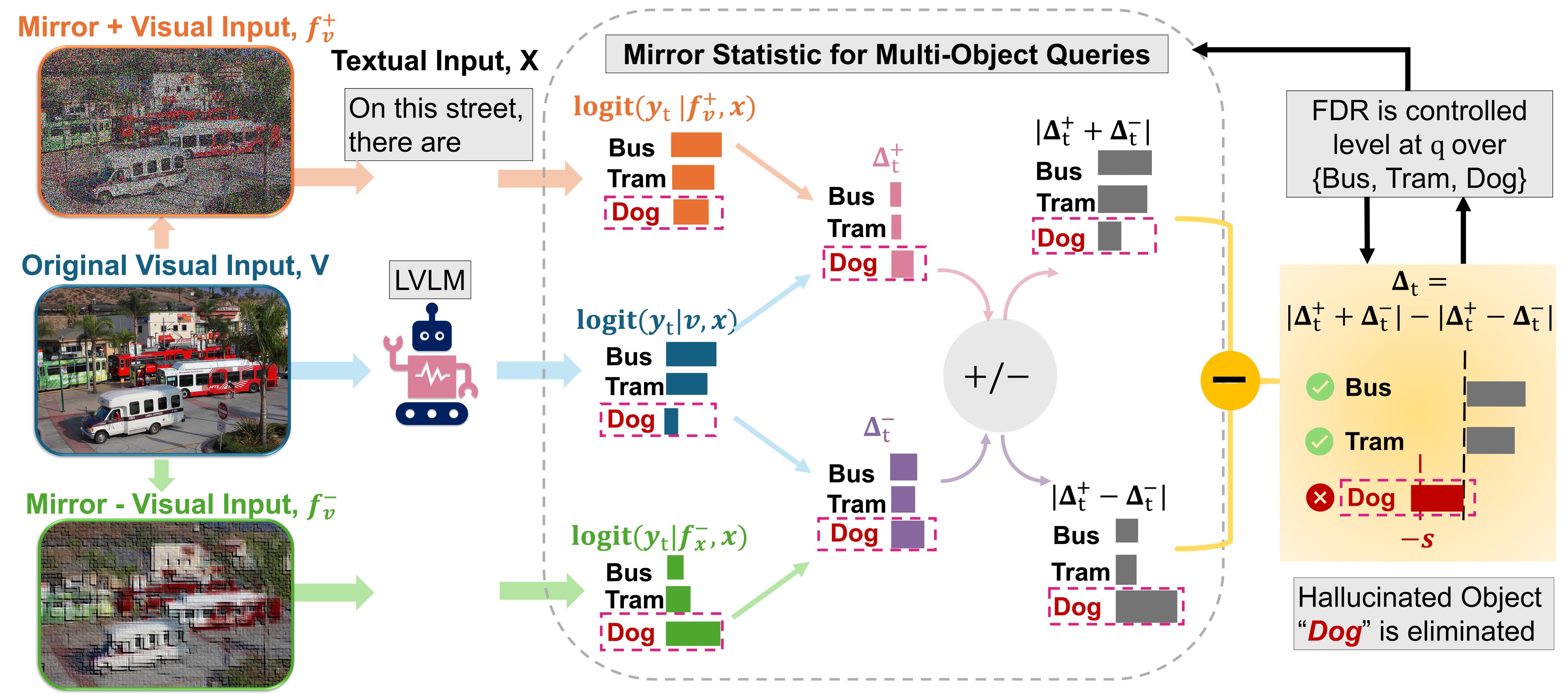}
    \caption{Illustration of $\method$, which applies visual uncertainty splitting and mirror statistic to control the FDR of hallucinated objects. The hallucinated object “Dog” (\textcolor{red}{red}) is removed under FDR control.
    }
    \label{fig:overview}
\end{figure*}

Object hallucination in LVLMs can be viewed as a failure to properly account for visual uncertainty during generation~\citep{zhou2023analyzing}. When visual inputs are reliable and unambiguous, token predictions are mainly driven by visual features extracted from the image~\citep{bigverdi2025perception}. In contrast, when visual evidence is weak or ambiguous, LVLMs tend to rely more on language priors, producing responses that are statistically plausible but not necessarily supported by the visual input~\citep{chang2025context}. As a result, objects with higher visual uncertainty are more likely to be hallucinated during the decoding process.

 \begin{wrapfigure}{r}{0.5\linewidth}
 \vspace{-0.25in}
    \centering
  \includegraphics[width=\linewidth]{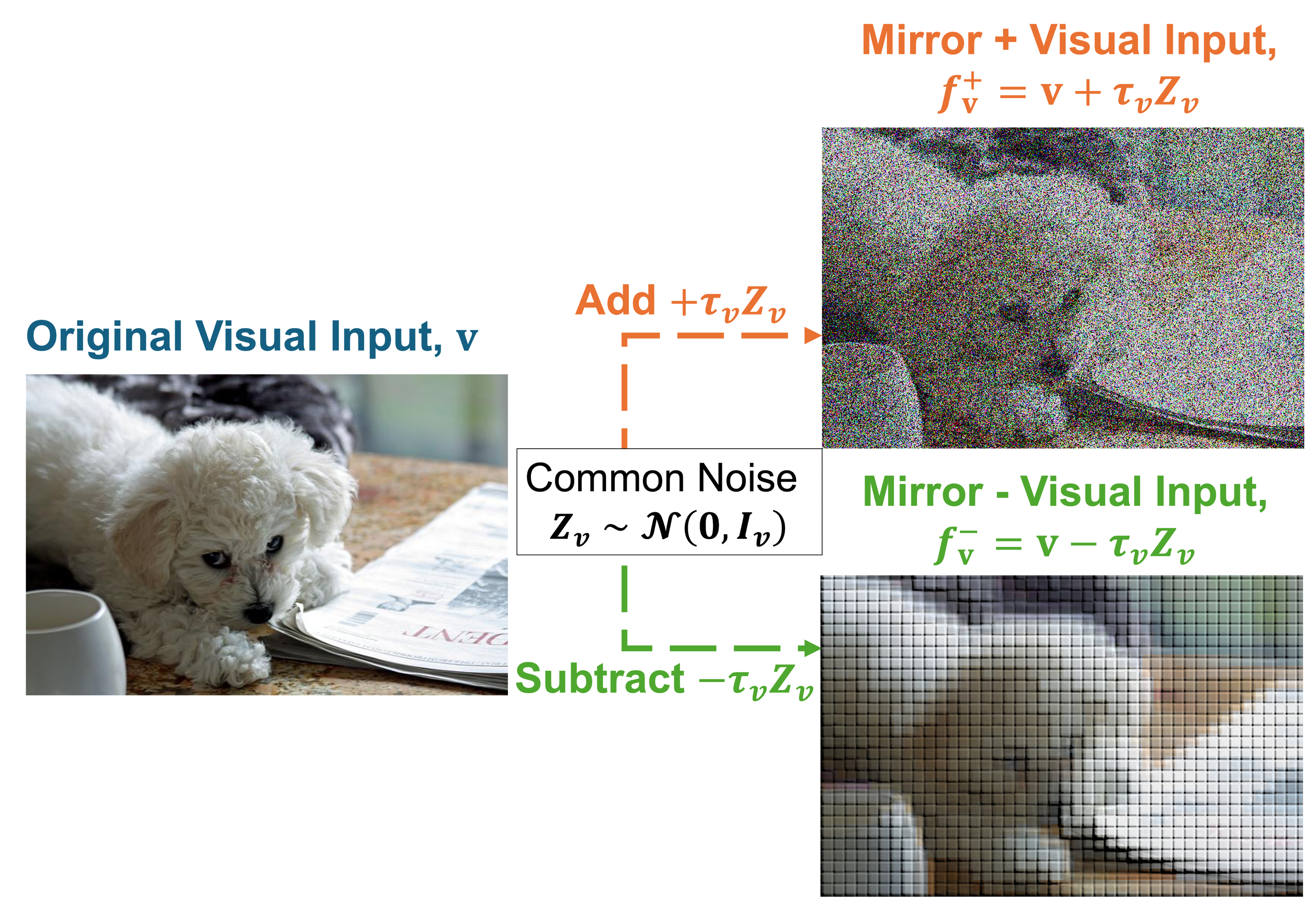}
\caption{\textbf{Uncertainty-aware visual data splitting}. Paired, symmetrically perturbed visual views are generated from a shared noise source, producing contrastive views for the mirror statistic construction and thereby facilitating the separation of visual signal from noise.
}
    \label{fig:ds}
     \vspace{-0.15in}
\end{wrapfigure}
Existing object hallucination mitigation methods largely rely on closed-set object existence queries, evaluating each object independently~\citep{lin2014microsoft, schwenk2022okvqa, hudson2019gqa}. Strategies such as image guidance~\citep{zhao2024mitigating} or global-local attention assembly~\citep{an2025mitigating} operate at the level of individual queries and require manually controlled signals. Contrastive approaches, such as Visual Contrastive Decoding (VCD)~\citep{leng2024mitigating}, probe visual uncertainty by comparing outputs from original and perturbed visual inputs, but still do not provide holistic, image-level control.
From a statistical perspective, contrastive visual analysis can be understood through the lens of \textit{data splitting}: the visual data are split into paired, symmetrically perturbed views from a shared noise
and then used in separate stages of the analysis to quantify uncertainty and control errors~\citep{dai2023false}.
As illustrated in Fig.~\ref{fig:ds}, we introduce an \textbf{uncertainty-aware visual data splitting} strategy that constructs two paired
visual inputs from each image via symmetric perturbations, by inducing controlled variability in the LVLM perception process. The splitting is instantiated in the visual domain by generating symmetric views from a shared noise source, allowing consistent visual signals to be distinguished from noise-induced variability.

Viewed in this way, visual hallucination detection can be cast as an uncertainty-aware objective selection problem, whose goal is to distinguish visually grounded outputs (e.g., \textit{fork}, \textit{broccoli}, and \textit{carrot} in Fig.~\ref{fig:multi_object}) from those that are not grounded in the
visual input (e.g., the hallucinated \textit{shrimp}).

Leveraging the symmetry of data splitting, we construct \textbf{\textit{mirror statistic}}~\citep{xing2023controlling}, a data-driven score for estimating false discoveries without access to ground truth. 
Built from two splitted visual data,
mirror statistic reward objectives that exhibit consistent visual evidence and confidence across splits, while inconsistent or noisy objectives tend to cancel out due to symmetric positive and negative contributions (e.g., the hallucinated \textit{Dog} in Fig.\ref{fig:overview}). This property enables direct estimation of spurious visual signals and provides a principled mechanism to mitigate object hallucination by estimating and controlling the false discovery rate (FDR), 
defined here at the level of an entire image as the expected proportion of selected visual objectives that are not visually grounded, rather than at the level of individual object queries, 
 \vspace{-0.15cm}
\begin{align}\label{eq:fdr}
\text{FDR}
= \mathbb{E} \left[
\frac{
\#\{\text{hallucinated objects identified as grounded}\}
}{
\#\{\text{all identified objects}\}
}
\right].
\end{align}
{e.g., four objects are predicted for the same Fig.~\ref{fig:multi_object}, one of which (\textit{shrimp}) is hallucinated, yielding an FDR of 25\%.} 

As a consequence of leveraging mirror statistic for FDR control, we bounds the expected fraction of falsely selected visual objects while maintaining high power, the probability of retaining truly visually grounded objects:
 \vspace{-0.2cm}
\begin{align}\label{eq:power}
\text{Power}
= \mathbb{E} \left[
\frac{\#\{\text{truly grounded objects correctly identified}\}}
     {\#\{\text{all truly grounded objects}\}}
\right].
\end{align}
In Fig.~\ref{fig:multi_object}, our proposed method successfully retains all visually grounded objects (\textit{fork}, \textit{broccoli}, and \textit{carrot}), achieving 100\% power for this image. The same figure also illustrates a common failure mode of LVLMs, where a hallucinated object (``\textit{shrimp}'') is predicted alongside genuine objects within a single image, allowing visually unsupported and grounded predictions to co-exist.
In this setting, false discovery rate (FDR) control limits the proportion of hallucinated objects retained in the output, while high power naturally follows by preserving objects that are genuinely supported by visual evidence rather than suppressing them indiscriminately. This balance enables effective image-level control of object hallucination.
Motivated by this observation, we propose $\method$, a false discovery rate \textbf{CO}nt\textbf{R}ol of object h\textbf{AL}lucination framework that controls the proportion of hallucinated objects at the image level while maintaining high power (Fig.~\ref{fig:overview}). Our main contributions are summarized as follows:
\begin{itemize}[leftmargin=*,nosep]
\item $\method$ employs an \textbf{uncertainty-aware visual data splitting} strategy to construct two paired visual views via symmetric perturbations of each image, inducing controlled variability and enabling stochastic contrasts for object hallucination control.
\item $\method$ leverages \textbf{mirror statistic} constructed from split visual inputs to estimate and control the false discovery rate at the image level, enabling principled suppression of hallucinated objects while retaining high power. 
\item $\method$ is a \textbf{training-free} approach that mitigates object hallucination through data-splitting–based FDR control, incurring low computational overhead compared to training-based alternatives.
\end{itemize}


\section{Related Work and Preliminaries} 
Controlling the false discovery rate is a fundamental problem in multiple hypothesis testing~\citep{benjamini1995controlling}. 
The knockoff framework~\citep{barber2015controlling} introduces synthetic variables to construct feature-level test statistics with provable false discovery rate (FDR) control for variable selection, and has been extended to high-dimensional settings through Model-X knockoffs~\citep{candes2018panning}. Similarly, Gaussian mirrors are developed for feature selection by constructing paired statistics that exhibit sign symmetry under the null hypothesis, such that positive and negative values are approximately balanced for null features~\citep{xing2023controlling}. This symmetry enables estimation of the number of false discoveries from the corresponding negative statistics, thereby facilitating FDR control~\citep{dai2023false, xing2023controlling, dai2023scale}. Both knockoff-based and mirror-statistic approaches are specifically designed for variable or feature selection tasks, relying on paired variables with exchangeability properties to control the selection FDR while maintaining statistical power, particularly in linear models.
Inspired by this line of work, we employ an uncertainty-aware visual data splitting strategy that constructs paired visual inputs, allowing mirror statistic to be applied for hallucination control in LVLMs. 


Existing approaches to object hallucination in LVLMs focus on either model adaptation or post hoc correction. 
Fine-tuning-based methods improve grounding using curated datasets~\citep{liu2023mitigating, xie2024v, awais2025foundation, zhou2023analyzing, gunjal2024detecting}, while others rely on external language models to refine outputs~\citep{zhai2023halle, yin2024woodpecker, li2026analyzing}. 
Recent work further suggests that hallucination primarily arises from the language modeling head, which fails to fully leverage visually grounded representations despite their presence in intermediate features~\citep{li2026analyzing}.
However, these approaches are often resource-intensive and difficult to scale, requiring large amounts of human annotation or reliance on proprietary models~\citep{tonmoy2024comprehensive,lin2024towards}. 
This has motivated lightweight alternatives, particularly uncertainty-based methods that detect hallucinations from model confidence or variability~\citep{huang2025survey, yang2023uncertainty}. Several recent approaches mitigate object hallucination by enhancing visual grounding through contrastive or multi-view visual signals during generation~\citep{xiu2026comprehensive}. Visual Contrastive Decoding (VCD)~\citep{leng2024mitigating} leverages contrastive visual inputs to discourage text generation unsupported by the image. Along a similar line, the Assembly of Global and Local Attention (AGLA)~\citep{an2025mitigating} constructs complementary visual representations from multiple views and evaluates token-level consistency across visual perspectives using contrastive signals, thereby reducing hallucinated outputs. Mitigating Hallucination via Image-Grounded Guidance (MARINE)~\citep{zhao2024mitigating} further extends this paradigm by incorporating multi-image reasoning and cross-modal alignment mechanisms to refine visual–textual consistency and suppress hallucination. Recent work also shows that hallucination can be exacerbated under distribution shifts such as stylized images, where visual cues deviate from natural image statistics and become less reliable~\citep{li2026saver}. As this setting focuses on stylized distribution shifts, it is not directly comparable to standard natural-image benchmarks used in our evaluation.



While effective, existing approaches largely rely on logit comparisons during decoding to suppress visually unsupported outputs, without explicit control over hallucination. 
This limitation is amplified in multi-object scenes with heterogeneous visual evidence. 
In contrast, our $\method$ formulates hallucination mitigation as an FDR-controlled selection problem, providing explicit image-level control over hallucinated objects while preserving power for grounded ones. 
We construct paired visual views via uncertainty splitting and apply mirror statistic to quantify consistent visual evidence, resulting in a principled, training-free framework for hallucination control.
\vspace{-0.1in}
\paragraph{Preliminaries: Decoding of LVLMs}
We consider a Large Vision-Language Model (LVLM) parameterized by $\theta$. 
The model takes a text prompt $\mathbf{x} = [x_1, \ldots, x_n]$, where each $x_i$ is a prompt token, and visual inputs $\mathbf{v}$, which provide contextual visual information. The model generates a response sequence $\mathbf{y} = [y_1, \ldots, y_m]$ autoregressively according to the conditional distribution $p_\theta(\mathbf{y} \mid \mathbf{v}, \mathbf{x})$, which factorizes as
\(
p_\theta(\mathbf{y} \mid \mathbf{v}, \mathbf{x}) = \prod_{t=1}^{m} p_\theta(y_t \mid \mathbf{v}, \mathbf{x}, \mathbf{y}_{<t}),
\)
where $\mathbf{y}_{<t} = [y_1, \ldots, y_{t-1}]$ for $t > 1$ and is empty for $t = 1$. 
At each step $t$, the next token is sampled from a categorical distribution defined by the model logits,
\(
p_\theta(y_t \mid \mathbf{v}, \mathbf{x}, \mathbf{y}_{<t}) = \mathrm{softmax}\!\left(\mathrm{logit}_\theta(\cdot \mid \mathbf{v}, \mathbf{x}, \mathbf{y}_{<t})\right).
\)
We can further view in the logit space, where the $t$-th token is sampled from the logit space by $\mathbf{y} \propto \exp \text{logit}_{\theta} (\mathbf{y} | \mathbf{v}, \mathbf{x}, \mathbf{y}_{<t})$.

In the decoding phase of LVLMs, object hallucination usually appears when probabilities are erroneously allocated to tokens that do not align with the presented visual input $\mathbf{v}$. The main causes of this problem include statistical biases inherent in training data~\citep{agarwal2020towards, agrawal2016analyzing, goyal2017making}, and over-reliance on language priors embedded within large language models (LLMs) used as decoders~\citep{gunjal2024detecting, li2023evaluating, yan2023overcoming, zhibo2023overcoming}. 



\section{$\method$ Methodology} \label{main:methodology}
To characterize and control object hallucination \emph{globally} at the image level rather than on individual object queries, we introduce \textbf{uncertainty-aware visual data splitting}. An LVLM typically produces multiple semantic outputs from a single image, each supported by varying degrees of visual evidence, motivating the need for global regulation of hallucinated content within the same decoding context. Our approach constructs paired, symmetrically perturbed views of the same visual input, inducing controlled visual uncertainty while preserving semantic content. This design disentangles visual dependence from language priors and enables assessment of whether token generation is genuinely grounded in visual evidence.
Building on this formulation, we cast hallucination mitigation as a false discovery rate (FDR) control problem using \textbf{mirror statistic}, a distribution-free procedure based on paired mirror comparisons that enables reliable image-level inference with principled error control over object hallucinations.
\subsection{Uncertainty-Aware Visual Data Splitting}
Visual uncertainty provides informative signals about the reliability of visual grounding in large vision language models rather than being merely a source of error~\citep{liu2024survey}. To leverage this information, we adopt a data splitting strategy~\citep{dai2023false} that introduces symmetric perturbations to the visual input. 
For each image, we apply an \emph{uncertainty-aware visual data splitting} procedure to construct a paired set of mirror views that capture visual uncertainty during decoding. Let $Z_v \sim \mathcal{N}(0, \mathbf{I}_v)$ be a Gaussian noise vector matching the dimensionality of the visual input $\mathbf{v}$. We generate two symmetrically perturbed visual features via sign-symmetric randomization: 
 \vspace{-0.1cm}
\begin{equation} \label{eq:ds}
    f_\mathbf{v}^{+} = \mathbf{v} + \tau_v Z_v, \qquad
    f_\mathbf{v}^{-} = \mathbf{v} - \tau_v Z_v ,
\end{equation}
where $\tau_v > 0$ controls the perturbation magnitude and can be calibrated on a held-out validation set for a target FDR. The calibrated $\tau_v$ is fixed for each visual backbone across all experiments, with sensitivity analyses provided in Appendix Fig.~\ref{fig:ablation_fdr_tau} and Fig.~\ref{fig:ablation_pope_tau}.
The resulting perturbed views preserve the same underlying visual content while introducing symmetric stochastic deviations along a shared random direction. 
The uncertainty-aware visual data splitting is designed to create a paired, sign-symmetric contrast that (i) amplifies visually grounded signals while canceling injected noise, and (ii) produces symmetric evidence patterns in the absence of visual grounding, a property that underlies the mirror statistic introduced in the next subsection and enables principled control of hallucinated objects.

\subsection{Mirror Statistic}\label{mirror}
Given the original visual input $\mathbf{v}$ and its two mirror views $f_\mathbf{v}^{+}$ and $f_\mathbf{v}^{-}$, we evaluate the LVLM under the same textual prompt $\mathbf{x}$.
During autoregressive decoding, for the generated output token $y_t$ at step $t$, we compare the model’s logits under the original and mirror view inputs. Specifically, we define the logit differences,
\begin{align}
    \Delta_t^{+}
    &= \text{logit}_{\theta}\!\left(y_t \mid \mathbf{v}, \mathbf{x}, \mathbf{y}_{<t}\right)
     - \text{logit}_{\theta}\!\left(y_t \mid f_\mathbf{v}^{+}, \mathbf{x}, \mathbf{y}_{<t}\right), \\
    \Delta_t^{-}
    &= \text{logit}_{\theta}\!\left(y_t \mid \mathbf{v}, \mathbf{x}, \mathbf{y}_{<t}\right)
     - \text{logit}_{\theta}\!\left(y_t \mid f_\mathbf{v}^{-}, \mathbf{x}, \mathbf{y}_{<t}\right).
\end{align}
These quantities capture how the model’s predictive behavior responds to sign-symmetric perturbations of the visual input along a shared noise realization. Since $f_\mathbf{v}^{+}$ and $f_\mathbf{v}^{-}$ preserve the same underlying visual content and differ only in perturbation sign, comparing $\Delta_t^{+}$ and $\Delta_t^{-}$ \textcolor{black}{logit contrasts that} reveal whether the model exhibits approximately symmetric responses to opposing visual perturbations.

Following the Gaussian mirrors construction \cite{xing2023controlling}, we define the mirror statistic for visual uncertainty
\begin{equation} \label{eq:mirror}
    \Delta_t
    =
    \bigl|\Delta_t^{+} + \Delta_t^{-}\bigr|
    -
    \bigl|\Delta_t^{+} - \Delta_t^{-}\bigr|.
\end{equation}

The mirror statistic $\Delta_t$ has two components. 
The first part, $\big|\Delta_t^{+} + \Delta_t^{-}\big|$, captures signal strength through consistent responses across symmetric perturbations, 
while the second part, $\big|\Delta_t^{+} - \Delta_t^{-}\big|$, reflects the noise cancellation effect. Their difference therefore provides a contrast between reliable visual evidence and perturbation-induced randomness. 

\textcolor{black}{Intuitively, when the generated token is not visually grounded,
its logit response to the two sign-symmetric visual perturbations
should not contain a stable visual-evidence component.
In this null case, the paired differences $\Delta_t^{+}$ and
$\Delta_t^{-}$ fluctuate primarily due to perturbation-induced
noise, and the resulting mirror statistic is expected to have
an approximately symmetric distribution around zero, rather
than being always negative.
In contrast, when the generated token is visually grounded,
the two mirror-view differences are expected to share a stable
visual-response component.
This component is reinforced in
$\left|\Delta_t^{+} + \Delta_t^{-}\right|$
and reduced in
$\left|\Delta_t^{+} - \Delta_t^{-}\right|$,
so visually grounded tokens tend to produce larger, more
positive values of $\Delta_t$.
As a result, $\Delta_t$ separates visually grounded outputs from noise-driven ones, enabling false discovery estimation from the negative tail of mirror statistic and principled FDR control with high power for true visual signals.}

\subsection{Controlling Hallucinations via FDR} \label{subsec:FDR}
We utilize the false discovery rate to control hallucinated tokens among the outputs generated by LVLMs. By leveraging the symmetric behavior of mirror statistic under sign-perturbed visual inputs, we derive a data-driven threshold that uses negative mirror evidence to estimate spurious visual responses, enabling FDR-controlled selection while maintaining high power to retain truly visually grounded tokens. 

Let $\mathcal{S}$ denote the set of decoding tokens generated for a given visual input, where each token $y_t \in \mathcal{S}$ is associated with mirror statistic $\Delta_t$, and let $\mathcal{S}_0 \subseteq \mathcal{S}$ denote the null subset whose generated token is not reliably grounded in the visual input (i.e., hallucinated tokens). Each token thus corresponds to a hypothesis testing whether it is visually grounded. 
A false discovery occurs when a token $\ y_t \in \mathcal{S}_0$  (\textit{i.e.}, a hallucinated token not reliably grounded in the visual input) is incorrectly identified as visually grounded by rejecting its null hypothesis.



\textcolor{black}{we state the required condition directly
as approximate conditional sign-flip invariance
under the null. Define
\[
H_{0t}:
\text{no systematic visual evidence supports token } y_t,
\]
and let
\[
\mathcal{G}_t
=
\sigma(\mathbf{v},\mathbf{x},\mathbf{y}_{<t})
\]
denote the shared context consisting of the image,
textual input, and decoding history, conditional on which $Z_v$ remains random.
\begin{theorem}[Conditional null symmetry]
\label{thm1}
For $\Delta_t$ defined in Eq.~\eqref{eq:mirror}, if
\[
(\Delta_t^+,\Delta_t^-)\mid H_{0t},\mathcal{G}_t
\overset{d}{=}
(\Delta_t^+,-\Delta_t^-)\mid H_{0t},\mathcal{G}_t,
\]
then, for every $s>0$,
\begin{equation}
\Pr(\Delta_t\le-s\mid H_{0t},\mathcal{G}_t)
=
\Pr(\Delta_t\ge s\mid H_{0t},\mathcal{G}_t).
\end{equation}
\end{theorem}
The proof follows from mirror antisymmetry
(Lemma~\ref{lemma1}).
For LVLMs, sign-flip invariance is an approximate
working assumption, not guaranteed by Gaussian
perturbations; uncorrelated contrasts are not required
(Appendix~\ref{app:symmetry}).
With suitable tail-count concentration
(Appendix~\ref{app:dependence}), approximate null
symmetry motivates, for $s>0$,
$
\#\{y_t\in\mathcal{S}_0:\Delta_t\ge s\}
\approx
\#\{y_t\in\mathcal{S}_0:\Delta_t\le-s\}
\le
\#\{y_t\in\mathcal{S}:\Delta_t\le-s\}.
$}


The $\mathrm{FDR}(s)$ at threshold $s$ can be estimated by
\begin{equation}\label{eq:fdr_math}
    \widehat{\mathrm{FDR}}(s)
    = \mathbb{E} \left[
    \frac{\#\{ y_t \mid \Delta_t \le -s \}}
         {\#\{ y_t \mid \Delta_t \ge s \} \vee 1} \right].
\end{equation}
Intuitively, the negative tail provides a data-driven estimate of how many hallucinated tokens are expected among the selected ones.

For any designated target FDR level $q \in (0,1)$, we choose a data-driven cutoff $T_q$ such that, 
\begin{equation}
    T_q
    =
    \min_s \big\{\widehat{\mathrm{FDR}}(s) \le q \big\},
\end{equation}
and retain the set $\{ y_t \in \mathcal{S} : \Delta_t \ge T_q \}$, suppressing the remaining tokens to mitigate hallucination.


Because mirror statistic are constructed independently for each decoding step, this design offers two advantages: (i) small, controllable perturbations reduce spurious correlations and improve power for identifying truly grounded tokens; and (ii) the computation is fully parallelizable across tokens, enabling scalability to outputs with many candidate tokens.

\section{Experiments}
\subsection{Experimental Setting}
\noindent\textbf{Models.} To demonstrate the broad applicability of our method across different LVLMs architectures, we apply and evaluate $\method$ to widely used models, including LLaVA-OneVision-7B~\citep{li2024llava}, Qwen2.5-VL-7B~\citep{bai2025qwen25vl}, and InternVL3-8B~\citep{zhu2025internvl3} (Table~\ref{tab:mlt_coco}), as well as LLaVA-v1.5~\citep{liu2023visual}, InstructBLIP-7B~\citep{dai2023instructblip}, and Qwen-VL-7B~\citep{bai2023Qwen} (Appendix~\ref{app:exp}).

\noindent\textbf{Dataset and Baselines.} In alignment with established evaluations from previous studies~\citep{leng2024mitigating}, we assess our method using the three dataset, MSCOCO \cite{lin2014microsoft}, A-OKVQA \cite{schwenk2022okvqa}, and GQA \cite{hudson2019gqa}. Each dataset includes three negative sample settings, i.e. random, popular, and adversarial. We compare our $\method$ with three state-of-the-art decoding methods and vanilla LVLMs without decoding techniques (regular), including VCD \cite{leng2024mitigating}, which distorts the image inputs to impose penalties on logit outputs; MARINE \cite{zhao2024mitigating}, which introduces image-grounded guidance, and AGLA \cite{an2025mitigating}, which assembles global features for response generation and local features for visual discrimination simultaneously. 
Results for MSCOCO are presented here, whereas results for A-OKVQA and GQA appear in Appendix~\ref{app:exp}.

\noindent\textbf{Metrics.} We evaluate all methods using FDR (Eq.~\eqref{eq:fdr}), power (Eq.~\eqref{eq:power}), and commonly used metrics including POPE and MME. We also report CHAIR metrics in the Appendix~\ref{app:exp}. Details are provided in Appendix~\ref{app:object_span}. 

Although $\method$ is primarily designed for object hallucination,  its token-level formulation, which operates directly on generated tokens without assuming a specific hallucination type, naturally generalizes to other types of hallucination, including attribute and relational hallucinations, as further validated empirically by MME results.

\textit{Polling-based Object Probing Evaluation} (POPE) \cite{li2023evaluating} evaluates closed-set choices. Multiple prompts with binary choices are formulated, each asking whether a specific object appears in the given image, such as ``Is there a chair in this image?" to answer ``yes" or ``no".  

\begin{table*}[ht]
\centering
\caption{Evaluation of overall FDR control and overall power across multiple LVLM architectures on MSCOCO (3000 runs). Overall FDR and power denote averages of per-image FDR and power. \textbf{Bold} indicates the best result and \uline{underline} indicates the second-best result. 
}
\label{tab:mlt_coco}
\setlength{\tabcolsep}{3pt}
\renewcommand{\arraystretch}{1.15}
\resizebox{\textwidth}{!}{%
\begin{tabular}{lcccccccc}
\toprule
\multirow{2}{*}{Method}
& \multicolumn{2}{c}{LLaVA-OneVision-7B}
& \multicolumn{2}{c}{Qwen2.5-VL-7B}
& \multicolumn{2}{c}{InternVL3-8B}
& \multicolumn{2}{c}{\textcolor{red}{Average}} \\
\cmidrule(lr){2-3}\cmidrule(lr){4-5}\cmidrule(lr){6-7}\cmidrule(lr){8-9}
& Overall FDR$\downarrow$ & Overall Power$\uparrow$
& Overall FDR$\downarrow$ & Overall Power$\uparrow$
& Overall FDR$\downarrow$ & Overall Power$\uparrow$
& Overall FDR$\downarrow$ & Overall Power$\uparrow$ \\
\midrule

\multicolumn{9}{l}{\textit{Random}}\\
\midrule
Regular
& 0.0935 ($\pm$ 0.0133) & 75.36 ($\pm$ 1.88)
& 0.0944 ($\pm$ 0.0114) & 77.43 ($\pm$ 1.22)
& 0.0946 ($\pm$ 0.0122) & 73.52 ($\pm$ 1.77)
& 0.0942 ($\pm$ 0.0111) & 75.44 ($\pm$ 1.44) 
\\
VCD (\citeyear{leng2024mitigating})
& 0.0892 ($\pm$ 0.0188) & 78.50 ($\pm$ 1.23)
& 0.0908 ($\pm$ 0.0132) & 79.14 ($\pm$ 1.53)
& 0.0899 ($\pm$ 0.0109) & 78.86 ($\pm$ 1.34)
& 0.0900 ($\pm$ 0.0145) & 78.83 ($\pm$ 1.31) \\
MARINE (\citeyear{zhao2024mitigating})
& 0.0846 ($\pm$ 0.0122) & 81.25 ($\pm$ 1.17)
& 0.0866 ($\pm$ 0.0116) & 80.05 ($\pm$ 1.60)
& 0.0889 ($\pm$ 0.0133) & 79.84 ($\pm$ 1.33)
& 0.0867 ($\pm$ 0.0124) & 80.38 ($\pm$ 1.21) \\
AGLA (\citeyear{an2025mitigating})
& \uline{0.0799 ($\pm$ 0.0223)} & \uline{84.92 ($\pm$ 1.11)}
& \uline{0.0821 ($\pm$ 0.0114)} & \uline{84.67 ($\pm$ 1.41)}
& \uline{0.0810 ($\pm$ 0.0144)} & \uline{80.92 ($\pm$ 1.33)}
& \uline{0.0810 ($\pm$ 0.0123)} & \uline{83.50 ($\pm$ 1.50)} \\
\rowcolor{blue!8}
\method (Ours)
& \textbf{0.0691 ($\pm$ 0.0113)} & \textbf{91.43 ($\pm$ 1.42)}
& \textbf{0.0718 ($\pm$ 0.0332)} & \textbf{94.60 ($\pm$ 1.49)}
& \textbf{0.0799 ($\pm$ 0.0122)} & \textbf{86.88 ($\pm$ 1.12)}
& \textbf{0.0736 ($\pm$ 0.0149)} & \textbf{90.97 ($\pm$ 1.44)} 
\\
\midrule

\multicolumn{9}{l}{\textit{Popular}}\\
\midrule
Regular
& 0.0971 ($\pm$ 0.0205) & 80.11 ($\pm$ 1.28)
& 0.0903 ($\pm$ 0.0111) & 77.01 ($\pm$ 1.64)
& 0.0937 ($\pm$ 0.0106) & 50.80 ($\pm$ 1.69)
& 0.0937 ($\pm$ 0.0141) & 69.31 ($\pm$ 1.54) 
\\
VCD (\citeyear{leng2024mitigating})
& 0.0893 ($\pm$ 0.0187) & 83.11 ($\pm$ 1.32) 
& 0.0911 ($\pm$ 0.0133) & 75.10 ($\pm$ 1.26) 
& 0.0875 ($\pm$ 0.0211) & 67.22 ($\pm$ 1.11) 
& 0.0893 ($\pm$ 0.0177) &  75.14 ($\pm$ 1.23) \\
MARINE (\citeyear{zhao2024mitigating})
& 0.0773 ($\pm$ 0.0134) & 85.22 ($\pm$ 1.55) 
& 0.0883 ($\pm$ 0.0220) & 78.28 ($\pm$ 1.25) 
& \uline{0.0701 ($\pm$ 0.0221)} & \uline{87.20 ($\pm$ 1.11)} 
& 0.0786 ($\pm$ 0.0192) & 83.57 ($\pm$ 1.30) \\
AGLA (\citeyear{an2025mitigating})
& \uline{0.0733 ($\pm$ 0.0177)} & \uline{88.13 ($\pm$ 1.69)} 
& \uline{0.0861 ($\pm$ 0.0233)} & \uline{82.22 ($\pm$ 1.51)} 
& 0.0741 ($\pm$ 0.0188) & 81.14 ($\pm$ 1.55) 
& \uline{0.0778 ($\pm$ 0.0199)} & \uline{83.83 ($\pm$ 1.58)} \\
\rowcolor{blue!8}
\method (Ours)
& \textbf{0.0552 ($\pm$ 0.0122)} & \textbf{90.42 ($\pm$ 0.24)} 
& \textbf{0.0802 ($\pm$ 0.0200)} & \textbf{84.10 ($\pm$ 1.15)} 
& \textbf{0.0675 ($\pm$ 0.0155)} & \textbf{88.93 ($\pm$ 1.21)} 
& \textbf{0.0676 ($\pm$ 0.0159)} & \textbf{87.82 ($\pm$ 0.87)} \\
\midrule

\multicolumn{9}{l}{\textit{Adversarial}}\\
\midrule
Regular
& 0.0988 ($\pm$ 0.0114) & 51.22 ($\pm$ 1.55) 
& 0.0973 ($\pm$ 0.0108) & 52.19 ($\pm$ 1.13) 
& 0.0953 ($\pm$ 0.0133) & 54.88 ($\pm$ 1.97) 
& 0.0971 ($\pm$ 0.0118) & 52.76 ($\pm$ 1.55) 
\\
VCD (\citeyear{leng2024mitigating})
& 0.0903 ($\pm$ 0.0115) & 56.32 ($\pm$ 1.56) 
& 0.0935 ($\pm$ 0.0109) & 53.16 ($\pm$ 1.77) 
& 0.0901 ($\pm$ 0.0131) & 60.59 ($\pm$ 1.26) 
& 0.0913 ($\pm$ 0.0118) & 56.69 ($\pm$ 1.53)  \\
MARINE (\citeyear{zhao2024mitigating})
& 0.0864 ($\pm$ 0.0105) & 61.36 ($\pm$ 1.99) 
& 0.0903 ($\pm$ 0.0117) & 60.22 ($\pm$ 1.53) 
& 0.0896 ($\pm$ 0.0156) & 70.46 ($\pm$ 1.59) 
& 0.0888 ($\pm$ 0.0126) & 64.01 ($\pm$ 1.70)  \\
AGLA (\citeyear{an2025mitigating})
& \uline{0.0775 ($\pm$ 0.0198)} & \uline{75.77 ($\pm$ 1.28)} 
& \uline{0.0899 ($\pm$ 0.0155)} & \uline{73.99 ($\pm$ 1.11)} 
& \uline{0.0826 ($\pm$ 0.0211)} & \uline{81.36 ($\pm$ 1.78)} 
& \uline{0.0833 ($\pm$ 0.0188)} & \uline{77.04 ($\pm$ 1.39)}  \\
\rowcolor{blue!8}
\method (Ours)
& \textbf{0.0694 ($\pm$ 0.0155)} & \textbf{88.33 ($\pm$ 1.45)} 
& \textbf{0.0839 ($\pm$ 0.0111)} & \textbf{83.63 ($\pm$ 1.55)} 
& \textbf{0.0733 ($\pm$ 0.0122)} & \textbf{87.59 ($\pm$ 1.33)} 
& \textbf{0.0755 ($\pm$ 0.0129)} & \textbf{86.52 ($\pm$ 1.44)}  \\
\bottomrule
\end{tabular}
}
\end{table*}

\textit{MME} evaluates the perception and cognition abilities of LVLMs~\citep{fu2025mme}, including object existence, attributes, spatial position, and relations. 

\textcolor{black}{
\textit{MMBench} evaluates the fine-grained multimodal understanding capabilities of large vision-language models through multiple-choice questions spanning 20 ability dimensions~\citep{liu2024mmbench}.
}

\noindent\textbf{Implementation Details.} Throughout all experiments, we followed the recommended settings from the respective papers and used the released codes to ensure fair comparisons. 
All evaluations were repeated over 3000 randomized trials (details in Appendix~\ref{app:implementation}).
For FDR control, we set the target level to $q = 0.1$, meaning that the expected proportion of hallucinated objects among all selected objects is controlled to be no greater than $10\%$.

\subsection{Experimental Results} \label{main:exp}
\noindent\textbf{Results on FDR \& Power.} 
The FDR and power results in Table~\ref{tab:mlt_coco} demonstrate that $\method$ achieves effective and consistent control of object hallucination under the target FDR level of $10\%$. Across most settings, $\method$ attains the lowest or near-lowest false discovery rate, indicating strong control over hallucinated objects compared to existing baselines. Importantly, this improved FDR control does not sacrifice power; $\method$ consistently achieves the highest or near-highest power, reflecting its ability to retain truly visually grounded objects while mitigating hallucinations. This favorable balance between multiple-testing error control (via FDR) and power is maintained on the more challenging \textit{popular} and \textit{adversarial} subsets, where competing methods exhibit noticeable degradation in power. These results demonstrate that $\method$ enables effective FDR control while maintaining high power, leading to robust hallucination mitigation.

\begin{table*}[ht]
\centering
\caption{Evaluation with POPE score across multiple LVLM architectures on the MSCOCO dataset.
We report individualized Accuracy and F1 score (mean $\pm$ std over 3000 runs).
\textbf{Bold} indicates the best result and \uline{underline} indicates the second-best result.
}
\label{tab:pope_coco}
\setlength{\tabcolsep}{3pt}
\renewcommand{\arraystretch}{1.15}
\small
\resizebox{\textwidth}{!}{%
\begin{tabular}{lcccccccc}
\toprule
\multirow{3}{*}{Method}
& \multicolumn{2}{c}{LLaVA-OneVision-7B}
& \multicolumn{2}{c}{Qwen2.5-VL-7B}
& \multicolumn{2}{c}{InternVL3-8B}
& \multicolumn{2}{c}{\textcolor{red}{Average}} \\
\cmidrule(lr){2-3}\cmidrule(lr){4-5}\cmidrule(lr){6-7}\cmidrule(lr){8-9}
& Accuracy$\uparrow$ & F1 Score$\uparrow$
& Accuracy$\uparrow$ & F1 Score$\uparrow$
& Accuracy$\uparrow$ & F1 Score$\uparrow$
& Accuracy$\uparrow$ & F1 Score$\uparrow$ \\
\midrule

\multicolumn{9}{l}{\textit{Random}}\\
\midrule
Regular
& 85.87 ($\pm$ 1.77) & 85.72 ($\pm$ 1.66)
& 86.20 ($\pm$ 1.30) & 86.98 ($\pm$ 1.28)
& 87.15 ($\pm$ 1.88) & 86.26 ($\pm$ 1.80)
& 86.41 ($\pm$ 1.65) & 86.32 ($\pm$ 1.58) \\
VCD (\citeyear{leng2024mitigating})
& 86.33 ($\pm$ 1.23) & 88.86 ($\pm$ 1.43)
& 87.35 ($\pm$ 1.83) & 90.06 ($\pm$ 2.03)
& 88.15 ($\pm$ 1.53) & 88.05 ($\pm$ 1.45)
& 87.28 ($\pm$ 1.53) & 88.99 ($\pm$ 1.64) \\
MARINE (\citeyear{zhao2024mitigating})
& 87.21 ($\pm$ 1.35) & 89.09 ($\pm$ 1.09)
& 89.72 ($\pm$ 1.11) & 90.33 ($\pm$ 2.09)
& 89.23 ($\pm$ 1.35) & 89.15 ($\pm$ 1.19)
& 88.72 ($\pm$ 1.27) & 89.52 ($\pm$ 1.46) \\
AGLA (\citeyear{an2025mitigating})
& \uline{88.15 ($\pm$ 1.65)} & \uline{89.97 ($\pm$ 1.21)}
& \uline{90.01 ($\pm$ 1.23)} & \uline{91.17 ($\pm$ 1.51)}
& \uline{92.53 ($\pm$ 1.42)} & \uline{93.54 ($\pm$ 1.77)}
& \uline{90.23 ($\pm$ 1.43)} & \uline{91.56 ($\pm$ 1.50)} \\
\rowcolor{blue!8}
$\method$ (Ours)
& \textbf{92.17 ($\pm$ 0.89)} & \textbf{91.64 ($\pm$ 1.83)}
& \textbf{90.63 ($\pm$ 1.54)} & \textbf{91.89 ($\pm$ 1.83)}
& \textbf{93.55 ($\pm$ 1.52)} & \textbf{95.35 ($\pm$ 1.87)}
& \textbf{92.12 ($\pm$ 1.32)} & \textbf{92.96 ($\pm$ 1.84)} \\
\midrule

\multicolumn{9}{l}{\textit{Popular}}\\
\midrule
Regular
& 82.72 ($\pm$ 1.99) & 83.16 ($\pm$ 1.85)
& 83.88 ($\pm$ 1.48) & 84.79 ($\pm$ 1.52)
& 84.63 ($\pm$ 1.48) & 85.26 ($\pm$ 1.15)
& 83.74 ($\pm$ 1.65) & 84.40 ($\pm$ 1.51) \\
VCD (\citeyear{leng2024mitigating})
& 84.24 ($\pm$ 1.46) & 86.16 ($\pm$ 1.30)
& 86.32 ($\pm$ 1.83) & 89.34 ($\pm$ 2.04)
& 85.29 ($\pm$ 1.14) & 86.25 ($\pm$ 2.00)
& 85.28 ($\pm$ 1.48) & 87.25 ($\pm$ 1.78) \\
MARINE (\citeyear{zhao2024mitigating})
& 86.81 ($\pm$ 1.21) & 88.19 ($\pm$ 2.09)
& \uline{88.72 ($\pm$ 1.38)} & 90.33 ($\pm$ 1.18)
& 87.67 ($\pm$ 1.53) & 88.37 ($\pm$ 1.81)
& 87.73 ($\pm$ 1.37) & 88.96 ($\pm$ 1.69) \\
AGLA (\citeyear{an2025mitigating})
& \uline{89.28 ($\pm$ 1.11)} & \uline{89.34 ($\pm$ 1.70)}
& 88.14 ($\pm$ 1.23) & \uline{90.79 ($\pm$ 1.15)}
& \uline{89.42 ($\pm$ 1.32)} & \uline{90.17 ($\pm$ 1.22)}
& \uline{88.95 ($\pm$ 1.22)} & \uline{90.10 ($\pm$ 1.36)} \\
\rowcolor{blue!8}
$\method$ (Ours)
& \textbf{90.37 ($\pm$ 1.89)} & \textbf{89.91 ($\pm$ 1.33)}
& \textbf{90.00 ($\pm$ 1.45)} & \textbf{91.28 ($\pm$ 1.83)}
& \textbf{91.33 ($\pm$ 1.35)} & \textbf{93.57 ($\pm$ 1.44)}
& \textbf{90.57 ($\pm$ 1.56)} & \textbf{91.59 ($\pm$ 1.53)} \\
\midrule

\multicolumn{9}{l}{\textit{Adversarial}}\\
\midrule
Regular
& 80.28 ($\pm$ 2.21) & 81.02 ($\pm$ 2.07)
& 83.17 ($\pm$ 1.37) & 83.46 ($\pm$ 2.08)
& 84.59 ($\pm$ 1.33) & 84.96 ($\pm$ 1.80)
& 82.68 ($\pm$ 1.64) & 83.15 ($\pm$ 1.98) \\
VCD (\citeyear{leng2024mitigating})
& 81.21 ($\pm$ 1.89) & 82.96 ($\pm$ 1.82)
& 85.43 ($\pm$ 1.88) & 85.24 ($\pm$ 2.04)
& 85.88 ($\pm$ 1.65) & 86.99 ($\pm$ 1.77)
& 84.17 ($\pm$ 1.81) & 85.06 ($\pm$ 1.88) \\
MARINE (\citeyear{zhao2024mitigating})
& 85.26 ($\pm$ 2.11) & 85.29 ($\pm$ 2.09)
& 87.42 ($\pm$ 1.11) & 88.73 ($\pm$ 1.91)
& 87.25 ($\pm$ 1.75) & 88.11 ($\pm$ 1.16)
& 86.64 ($\pm$ 1.66) & 87.38 ($\pm$ 1.72) \\
AGLA (\citeyear{an2025mitigating})
& \uline{86.44 ($\pm$ 1.61)} & \textbf{88.14 ($\pm$ 2.01)}
& \uline{91.39 ($\pm$ 1.23)} & \uline{89.73 ($\pm$ 2.15)}
& \uline{90.11 ($\pm$ 1.32)} & \uline{89.83 ($\pm$ 1.10)}
& \uline{89.31 ($\pm$ 1.39)} & \uline{89.23 ($\pm$ 1.75)} \\
\rowcolor{blue!8}
$\method$ (Ours)
& \textbf{88.27 ($\pm$ 1.81)} & \uline{87.99 ($\pm$ 2.03)}
& \textbf{93.75 ($\pm$ 1.54)} & \textbf{90.87 ($\pm$ 1.83)}
& \textbf{91.03 ($\pm$ 1.00)} & \textbf{90.48 ($\pm$ 1.83)}
& \textbf{91.02 ($\pm$ 1.45)} & \textbf{89.78 ($\pm$ 1.90)} \\
\bottomrule
\end{tabular}
}
\end{table*}

\noindent\textbf{Results on POPE.} POPE evaluates object-level grounding in LVLMs by testing their ability to answer \textit{yes-or-no} questions about visual content. On the MSCOCO dataset, we report accuracy and F1 score. As shown in Table~\ref{tab:pope_coco}, $\method$ consistently improves performance across all evaluated LVLMs and settings, demonstrating its effectiveness in mitigating object hallucinations. Notably, the improvements are particularly evident under the more challenging \textit{adversarial} setting, where hallucination errors are more likely to occur. Additional results are provided in Tables~\ref{tab:pope_coco1_random}, \ref{tab:pope_vaq}, and \ref{tab:pope_gqa} in the Appendix. By explicitly controlling the false discovery rate (FDR), our method prioritizes limiting false positive predictions within each image. While this emphasis may introduce a mild trade-off with recall, it leads to more reliable predictions and improved overall performance in hallucination mitigation.

\begin{table*}[h]
\centering
\caption{Evaluation with MME score (mean $\pm$ std over 3000 runs) across multiple LVLM architectures on MSCOCO. 
\textbf{Bold} indicates the best result and \uline{underline} indicates the second-best result.}
\label{tab:mme_new}
\setlength{\tabcolsep}{4pt}
\renewcommand{\arraystretch}{1.15}
\resizebox{\textwidth}{!}{%
\begin{tabular}{llcccccc}
\toprule
\multirow{2}{*}{Model} & \multirow{2}{*}{Method}
& \multicolumn{2}{c}{Object}
& \multicolumn{1}{c}{Attribute}
& \multicolumn{2}{c}{Relation}
& \multirow{2}{*}{Total $\uparrow$} \\
\cmidrule(lr){3-4} \cmidrule(lr){5-5} \cmidrule(lr){6-7}
& 
& Existence $\uparrow$
& Count $\uparrow$
& Color $\uparrow$
& Position $\uparrow$
& Commonsense $\uparrow$
& \\
\midrule

\multirow{5}{*}{LLaVA-OneVision-7B}
& Regular
& 190.33 ($\pm$ 6.50) & 145.53 ($\pm$ 15.20) & 170.66 ($\pm$ 9.10) & 160.25 ($\pm$ 8.30) & 70.16 ($\pm$ 6.50) & 736.93 ($\pm$ 24.80) \\
& VCD (\citeyear{leng2024mitigating})
& 186.25 ($\pm$ 7.22) & 147.25 ($\pm$ 11.44) & 175.35 ($\pm$ 15.58) & 165.36 ($\pm$ 2.55) & 79.42 ($\pm$ 5.74) & 753.63 ($\pm$ 18.76) \\
& MARINE (\citeyear{zhao2024mitigating})
& \uline{191.26 ($\pm$ 4.55)} & 150.44 ($\pm$ 10.15) & 177.45 ($\pm$ 10.96) & 165.36 ($\pm$ 7.19) & 82.33 ($\pm$ 5.29) & 766.84 ($\pm$ 17.47) \\
& AGLA (\citeyear{an2025mitigating})
& 190.54 ($\pm$ 6.22) & \uline{153.32 ($\pm$ 15.37)} & \uline{178.42 ($\pm$ 3.22)} & \uline{170.35 ($\pm$ 4.22)} & \uline{84.36 ($\pm$ 5.21)} & \uline{776.99 ($\pm$ 15.96)} \\
\rowcolor{blue!8}
& $\method$ (Ours)
& \textbf{193.67 ($\pm$ 3.21)} & \textbf{160.43 ($\pm$ 10.11)} & \textbf{180.35 ($\pm$ 8.46)} & \textbf{175.24 ($\pm$ 10.58)} & \textbf{85.52 ($\pm$ 4.22)} & \textbf{795.21 ($\pm$ 18.12)} \\

\midrule

\multirow{5}{*}{Qwen2.5-VL-7B}
& Regular
& 185.52 ($\pm$ 5.80) & 140.46 ($\pm$ 12.20) & 165.16 ($\pm$ 8.60) & 155.22 ($\pm$ 7.50) & 80.53 ($\pm$ 6.20) & 726.89 ($\pm$ 21.30) \\
& VCD (\citeyear{leng2024mitigating})
& 187.35 ($\pm$ 7.36) & 145.63 ($\pm$ 10.45) & 173.35 ($\pm$ 8.23) & 160.22 ($\pm$ 11.34) & 82.55 ($\pm$ 6.29) & 749.10 ($\pm$ 11.61) \\
& MARINE (\citeyear{zhao2024mitigating})
& 189.83 ($\pm$ 4.25) & 153.35 ($\pm$ 10.24) & 176.25 ($\pm$ 7.21) & 165.14 ($\pm$ 10.44) & 84.21 ($\pm$ 4.87) & 768.78 ($\pm$ 11.88) \\
& AGLA (\citeyear{an2025mitigating})
& \uline{190.15 ($\pm$ 3.28)} & \uline{158.33 ($\pm$ 11.38)} & \uline{179.35 ($\pm$ 4.15)} & \uline{168.13 ($\pm$ 9.83)} & \uline{87.28 ($\pm$ 2.28)} & \uline{783.24 ($\pm$ 13.42)} \\
\rowcolor{blue!8}
& $\method$ (Ours)
& \textbf{194.24 ($\pm$ 3.18)} & \textbf{161.35 ($\pm$ 10.01)} & \textbf{182.53 ($\pm$ 6.98)} & \textbf{170.86 ($\pm$ 8.14)} & \textbf{88.91 ($\pm$ 3.71)} & \textbf{797.89 ($\pm$ 12.67)} \\

\midrule

\multirow{5}{*}{InternVL3-8B}
& Regular
& 192.25 ($\pm$ 6.20) & 150.22 ($\pm$ 10.80) & 175.48 ($\pm$ 7.40) & 165.16 ($\pm$ 6.50) & 95.38 ($\pm$ 5.20) & 778.49 ($\pm$ 19.80) \\
& VCD (\citeyear{leng2024mitigating})
& 193.26 ($\pm$ 3.28) & 159.25 ($\pm$ 10.97) & 178.25 ($\pm$ 9.25) & 175.35 ($\pm$ 7.73) & 95.15 ($\pm$ 1.54) & 801.26 ($\pm$ 13.36) \\
& MARINE (\citeyear{zhao2024mitigating})
& \textbf{193.87 ($\pm$ 2.82)} & 163.28 ($\pm$ 3.29) & 180.93 ($\pm$ 5.18) & 177.35 ($\pm$ 4.26) & \uline{96.18 ($\pm$ 1.58)} & 811.61 ($\pm$ 13.42) \\
& AGLA (\citeyear{an2025mitigating})
& 193.17 ($\pm$ 3.28) & \uline{165.26 ($\pm$ 7.77)} & \uline{183.85 ($\pm$ 5.28)} & \uline{177.96 ($\pm$ 4.26)} & 95.15 ($\pm$ 1.77) & \uline{815.39 ($\pm$ 12.81)} \\
\rowcolor{blue!8}
& $\method$ (Ours)
& \uline{193.46 ($\pm$ 4.26)} & \textbf{168.36 ($\pm$ 10.53)} & \textbf{185.29 ($\pm$ 13.19)} & \textbf{178.54 ($\pm$ 1.35)} & \textbf{96.28 ($\pm$ 1.98)} & \textbf{821.93 ($\pm$ 12.30)} \\
\bottomrule
\end{tabular}
}
\end{table*}

\noindent\textbf{Results on MME.} The MME evaluation extends beyond POPE by covering a broader range of hallucination types, including object-, attribute-, and relation-level hallucinations. As shown in Table~\ref{tab:mme_new}, $\method$ consistently improves overall MME performance across all evaluated LVLM architectures, achieving the best total scores in all settings. These results suggest that, although $\method$ is primarily designed for object-level hallucination control, it can also help mitigate attribute-level inconsistencies by suppressing noise-induced predictions that are not well supported by the visual input. This indicates that the benefits of $\method$ extend beyond object presence to more fine-grained visual descriptions. Additional MME results for other models are reported in Table~\ref{tab:mme} in Appendix.

\noindent\textbf{Results on MMBench.}
We further evaluate $\method$ on MMBench~\citep{liu2024mmbench} to assess its effect on general multimodal understanding beyond object-hallucination benchmarks. As shown in Table~\ref{tab:mmbench}, $\method$ improves the MMBench score across all three evaluated backbones. For LLaVA-OneVision-7B, the score increases from 80.8 to 86.7, yielding a gain of 5.9 points. For Qwen2.5-VL-7B and InternVL3-8B, the scores improve from 83.5 to 87.9 and from 83.4 to 94.5, respectively, corresponding to gains of 4.4 and 11.1 points. These results suggest that $\method$ mitigates object hallucination while improving broader multimodal understanding under the evaluated settings.

\begin{table*}[t]
\centering
\caption{Evaluation with MMBench score (mean $\pm$ std over 3000 runs) across multiple LVLM architectures on MSCOCO. 
\textbf{Bold} indicates the best result and \uline{underline} indicates the second-best result.}
\label{tab:mmbench}
\setlength{\tabcolsep}{10pt}
\renewcommand{\arraystretch}{1.15}
\begin{tabular}{lccc}
\toprule
Method
& LLaVA-OneVision-7B
& Qwen2.5-VL-7B
& InternVL3-8B \\
\midrule
Regular
& 80.77 ($\pm$ 1.21) & 83.50 ($\pm$ 1.01) & 83.40 ($\pm$ 0.85)\\
VCD (\citeyear{leng2024mitigating})
& 82.67 ($\pm$ 1.00) & 85.14 ($\pm$ 1.32) & 86.84 ($\pm$ 0.99)\\
MARINE (\citeyear{zhao2024mitigating})
& 85.31 ($\pm$ 1.09) & \uline{87.44 ($\pm$ 1.17)} & 88.67 ($\pm$ 1.03)\\
AGLA (\citeyear{an2025mitigating})
& \uline{85.76 ($\pm$ 1.11)} & 87.02 ($\pm$ 1.07)& \uline{90.15 ($\pm$ 1.00)}\\
\rowcolor{blue!8}
$\method$ (Ours)
& \textbf{86.70 ($\pm$ 1.13)} & \textbf{87.90 ($\pm$ 0.91)} & \textbf{94.50 ($\pm$ 0.80)} \\
\bottomrule
\end{tabular}
\end{table*}

\begin{table*}[htbp]
\centering
\caption{Ablation study on POPE metrics using the MSCOCO dataset with LLaVA-OneVision-7B.
Results are mean $\pm$ std over 10 runs.
\textbf{Bold} indicates the best result.}
\label{tab:ablation_pope}

\small
\begin{tabular}{lcccc}
\midrule
Method
& Accuracy $\uparrow$
& Precision $\uparrow$
& Recall $\uparrow$
& F1 Score $\uparrow$ \\
\midrule

Regular
& 85.87 ($\pm$ 1.77)
& 83.41 ($\pm$ 2.13)
& 88.08 ($\pm$ 1.47)
& 85.72 ($\pm$ 1.66) \\

w/o Visual Uncertainty Splitting
& 88.12 ($\pm$ 0.62)
& 86.78 ($\pm$ 0.74)
& 89.07 ($\pm$ 0.71)
& 87.90 ($\pm$ 0.58) \\

w/o Mirror Statistic
& 88.76 ($\pm$ 0.55)
& 87.07 ($\pm$ 0.49)
& 90.33 ($\pm$ 0.61)
& 88.67 ($\pm$ 0.53) \\

w/o Overall FDR Control
& 89.36 ($\pm$ 0.48)
& 86.42 ($\pm$ 0.92)
& 92.37 ($\pm$ 0.56)
& 89.27 ($\pm$ 0.41) \\

\midrule
\rowcolor{blue!8}
$\method$ (Ours)
& \textbf{92.17 ($\pm$ 0.89)}
& \textbf{88.25 ($\pm$ 1.46)}
& \textbf{94.87 ($\pm$ 1.61)}
& \textbf{91.64 ($\pm$ 1.83)} \\
\midrule

\end{tabular}
\end{table*}

\begin{wrapfigure}{r}{0.55\linewidth}
    \centering
    \vspace{-0.18in}
    \includegraphics[width=\linewidth]{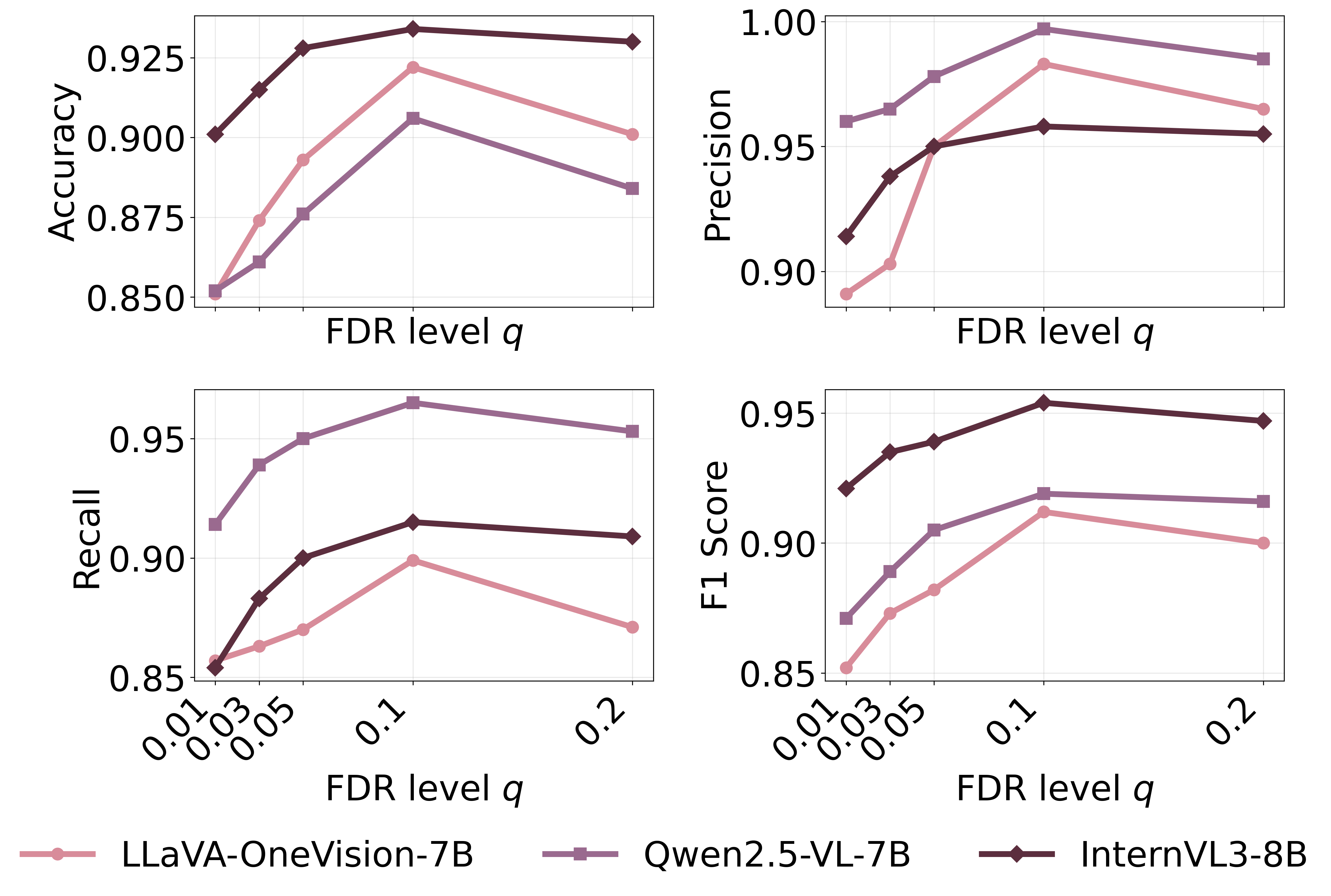}
    \caption{Ablation study on the effect of FDR target level ($q$) on the performance of LLaVA-OneVision-7B, Qwen2.5-VL-7B, InternVL3-8B using POPE metrics with $q = \{0.01, 0.03, 0.05, 0.1, 0.2\}$.}
    \label{fig:q_level_new}
    \vspace{-0.22in}
\end{wrapfigure}

\subsection{Ablation Study}
We conduct ablation studies to evaluate the contribution of visual uncertainty splitting, mirror statistic, and FDR control. As shown in Table~\ref{tab:ablation_pope}, removing any component degrades performance, confirming their complementary roles. In particular, disabling FDR control increases recall but reduces precision and F1, indicating a higher tendency to retain hallucinated objects. Removing mirror statistic leads to reduced recall, reflecting weaker ability to preserve truly grounded objects, while removing visual uncertainty splitting results in overall performance degradation. The full $\method$ achieves the best results across all metrics, demonstrating its effectiveness in balancing FDR control and power. Additional results are provided in Figures~\ref{fig:ablation_fdr_tau} and~\ref{fig:ablation_pope_tau} in the Appendix.


\noindent\textbf{Effect of FDR Control at Different Levels on Object Hallucinations.} Figure~\ref{fig:q_level_new} shows the impact of the FDR target level $q$ on performance across LVLMs. As $q$ increases from small values, both accuracy and F1 score improve, indicating that overly strict control suppresses not only hallucinated objects but also true positives. Performance peaks at intermediate values of $q$, where a favorable balance between hallucination mitigation and retention of grounded objects is achieved. Further increasing $q$ leads to diminishing returns or slight degradation, as more spurious predictions are admitted. These results highlight the fundamental FDR-power trade-off, with consistent trends observed across architectures. This behavior aligns with the theoretical role of FDR control in regulating false discoveries while preserving statistical power.



\noindent\textbf{Latency Analysis.}
We evaluate inference efficiency by measuring the average latency per generated token. Detailed results are provided in Figure~\ref{fig:latency} in the Appendix~\ref{app:latency}. Among all approaches, $\method$ incurs the smallest latency increase. 
Among all approaches, $\method$ incurs the smallest latency increase (50.26 ms/token). While AGLA~\citep{an2025mitigating} (51.42 ms/token) and MARINE~\citep{zhao2024mitigating} (52.21 ms/token) rely on iterative decoding or repeated sampling, and VCD~\citep{leng2024mitigating} (53.42 ms/token) performs contrastive decoding within the autoregressive loop, these methods introduce step-wise overhead that accumulates over the sequence. Although $\method$ evaluates three visual inputs, these forward passes are used only for post-hoc mirror statistic computation. As a result, $\method$ avoids decoding-time logit reweighting and achieves lower latency than VCD despite the additional view.

\begin{wrapfigure}{l}{0.55\linewidth}
    \centering
    \vspace{-0.19in}
    \includegraphics[width=\linewidth]{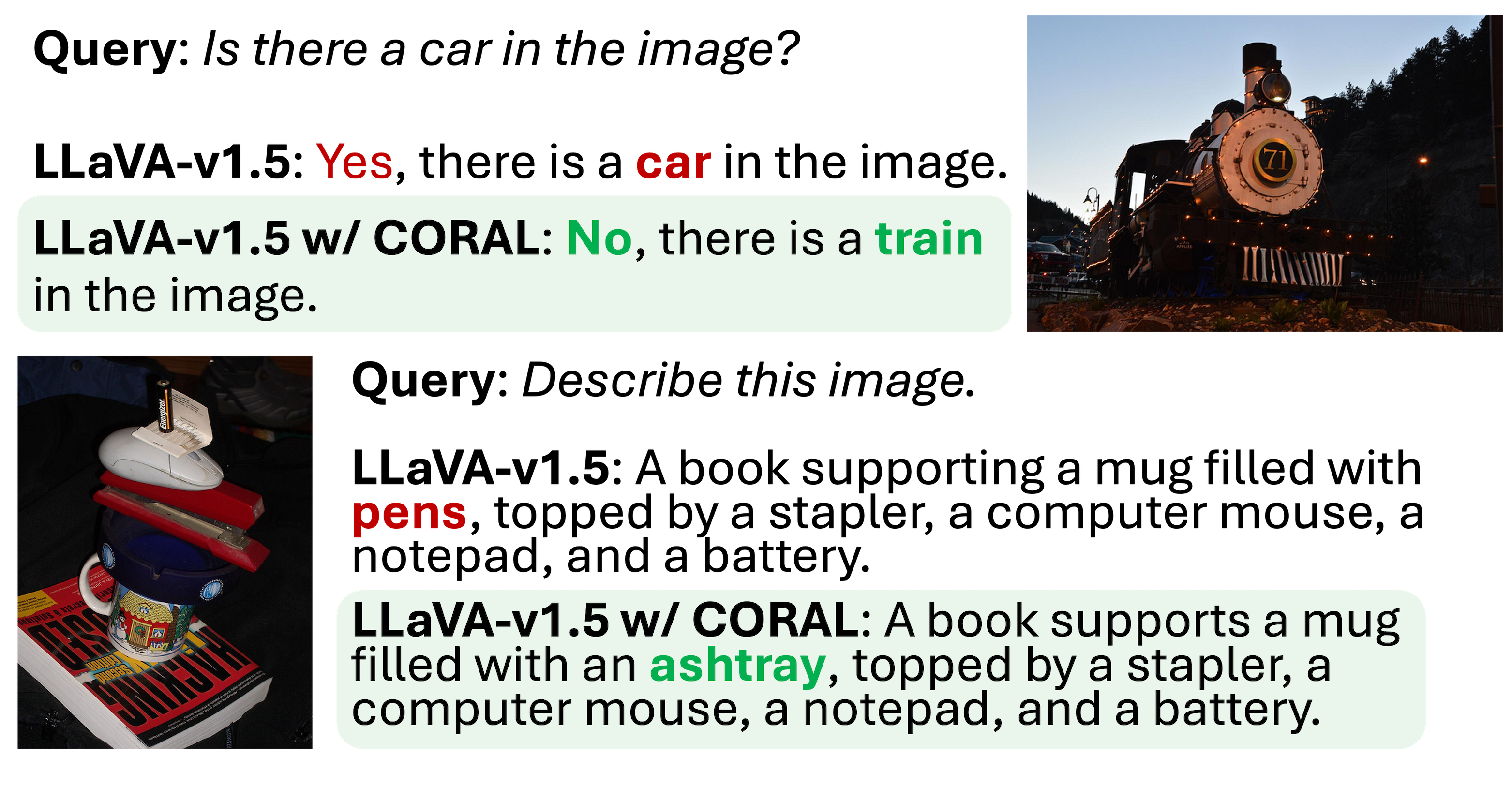}
    \caption{Hallucination mitigation examples by our method \method~ across multiple tasks. Hallucinated objects are highlighted in \textcolor{red}{red}.}
    \label{fig:example}
     \vspace{-0.15in}
\end{wrapfigure}
\section{Conclusion, Limitations, and Future Work} \label{main:conclusion}
We propose $\method$, a training-free framework that introduces visual uncertainty via data splitting and leverages mirror statistic to control the FDR of hallucinated objects during decoding. By explicitly controlling false positives at the image level while maintaining
high power, $\method$ suppresses hallucinations without discarding informative visual evidence. Extensive experiments across multiple LVLM architectures demonstrate consistent reductions in FDR and improvements in power and accuracy, particularly under challenging \textit{popular} and \textit{adversarial} settings.

\noindent\textbf{Limitations and Future Work.}  
Although effective, $\method$ cannot be applied directly to black-box commercial APIs that do not expose token-level logits. Future work includes extending the FDR-based framework to handle more complex hallucinations, such as relational and attribute errors, and to open-ended generation tasks beyond object existence queries. More broadly, integrating $\method$ with adaptive decoding strategies and applying FDR control to multimodal reasoning tasks are promising directions. 
\newpage
\bibliographystyle{unsrtnat}
\bibliography{CORAL}






\newpage
\appendix
\section{Impact Statement} \label{app:impact}
This paper presents research aimed at improving the reliability of large vision-language models by reducing hallucinated outputs during generation. The proposed method, $\method$, operates directly at decoding time and does not require additional training, external data, or model fine-tuning. This makes it easy to apply to existing models and practical for real-world use. By improving the alignment between generated outputs and visual input, $\method$ helps reduce visually ungrounded content while preserving useful and relevant information. This targeted reduction of hallucinations can improve the quality and dependability of model responses, especially in applications where incorrect outputs may mislead users. At the same time, $\method$ focuses specifically on hallucinations related to visual grounding and does not address other issues such as harmful language or biases inherited from model pretraining. These limitations are outside the scope of this work and may be explored in future research. To the best of our knowledge, this work does not introduce negative societal impacts associated with our research that merit highlighting in these discussions. 

\section{Experiment Setup} \label{app:exp}
We conduct all of the experiments using a cluster with four NVIDIA Quadro RTX 6000 (24GB) GPUs using CUDA 11.7. Each single experiment can be run on a single RTX 6000 GPU. 
\subsection{Model Architecture}
In Table \ref{tab:lvml_arch}, we provide detailed descriptions of the LVLM architectures used in our experiments. These LVLMs respectively leverage the pre-trained vision encoder of the models we listed, which are all based on the Vision Transformer (ViT)~\citep{dosovitskiy2020image} architecture. 

\begin{table}[h]
\centering
\caption{Details of the LVLM architectures used in our experiments.}
\label{tab:lvml_arch}
\begin{tabular}{l|c|c}
\hline
Model & Vision encoder & LLM \\
\hline
LLaVA-v1.5~\citep{liu2023visual}
& CLIP-L-336px~\citep{radford2021learning} 
& Vicuna-v1.5-7B~\citep{chiang2023vicuna} \\

Qwen-VL~\citep{bai2023Qwen} 
& ViT-based visual encoder 
& Qwen-7B~\citep{bai2023Qwen} \\

InstructBLIP~\citep{dai2023instructblip}
& BLIP-2~\citep{li2023blip}
& Vicuna-v1.1-7B~\citep{chiang2023vicuna} \\

LLaVA-OneVision-7B~\citep{li2024llava} & CLIP ViT-L/14 (336px)~\citep{radford2021learning} & Vicuna-7B~\citep{chiang2023vicuna} \\
Qwen2.5-VL-7B~\citep{bai2025qwen25vl} & ViT-based visual encoder & Qwen2.5-7B~\citep{bai2025qwen25vl} \\
InternVL3-8B~\citep{zhu2025internvl3} & InternViT (high-resolution ViT) & InternLM2-8B~\citep{zhu2025internvl3} \\
\hline
\end{tabular}
\end{table}

\subsection{Prompt Template}
For each query, we randomly select a prompt template from the available template list, as shown in Table~\ref{tab:prompt_templates_ours}.

\begin{table}[ht]
\centering
\caption{Details of the LVLM architectures that we used in our paper.}
\label{tab:prompt_templates_ours}
\begin{tabular}{p{0.28\linewidth} p{0.68\linewidth}}
\toprule
\textbf{Template Type} & \textbf{Prompt Template} \\
\midrule
POPE task
&
This image contains only the following objects: \texttt{<OBJECT\_GROUNDING>}.
Do not assume any objects beyond this list. Based solely on this information, \texttt{<QUERY>}

The detected objects in the image are: \texttt{<OBJECT\_GROUNDING>}.
Answer the question using only these objects. \texttt{<QUERY>}

This image shows the following objects: \texttt{<OBJECT\_GROUNDING>}. You must answer using only the objects in this list.Given these detected objects, \texttt{<QUERY>} 

The objects found in this image are limited to: \texttt{<OBJECT\_GROUNDING>}. You should rely strictly on this list of objects and make no other guesses. Based on this, \texttt{<QUERY>} \\
\midrule
$\method$ grounded
&
This image contains the following visually grounded objects: \texttt{<OBJECT\_GROUNDING>}.
Based on the image, \texttt{<QUERY>}

The following objects are visible in the image: \texttt{<OBJECT\_GROUNDING>}.
Using only the information from the image, \texttt{<QUERY>}

This image shows: \texttt{<OBJECT\_GROUNDING>}.
Please answer the following question based on the image content: \texttt{<QUERY>}
\\
\midrule
$\method$-restricted
&
The objects visible in this image are limited to: \texttt{<OBJECT\_GROUNDING>}.
Do not assume any objects beyond this list.
\texttt{<QUERY>}

Only the following objects appear in the image: \texttt{<OBJECT\_GROUNDING>}.
Answer the question using only visual evidence from the image.
\texttt{<QUERY>}

Based strictly on the objects shown in the image: \texttt{<OBJECT\_GROUNDING>}.
Do not infer any additional objects.
\texttt{<QUERY>}
\\
\midrule
$\method$-complementary
&
The same image is analyzed using multiple internally constructed visual representations, including the original view and two mirrored views. 

Original view detects the following objects: \texttt{<OBJECT\_GROUNDING>} 

Mirror view (+) detects the following objects: \texttt{<OBJECT\_GROUNDING\_A>} 

Mirror view (-) detects the following objects: \texttt{<OBJECT\_GROUNDING\_B>} 

Using the visual information above from the same image, \texttt{<QUERY>}

Multiple complementary visual representations are derived from the same image. 

Original view objects: \texttt{<OBJECT\_GROUNDING>} 

Mirrored view (+) objects: \texttt{<OBJECT\_GROUNDING\_A>} 

Mirrored view (-) objects: \texttt{<OBJECT\_GROUNDING\_B>} 

Based on the image, \texttt{<QUERY>}
\\
\bottomrule
\end{tabular}
\end{table}

\subsection{Implementation Details of Visual Perturbations}
In Eq.~\eqref{eq:ds}, the visual input $\mathbf{v}$ refers to the
patch-level visual embeddings produced by the vision encoder. Specifically, given a patch-level visual embeddings $\mathbf{v}$, we construct mirror views as $\mathbf{v}^{\pm} = \mathbf{v} \pm \tau_v \mathbf{Z}_v$,
where $\mathbf{Z}_v \sim \mathcal{N}(0, \mathbf{I})$ is sampled independently
per image. This design ensures that both mirror views share the same visual semantics while inducing controlled uncertainty in the visual signal.

\subsection{Empirical Validation of the Mirror-Symmetry Property}
\label{app:symmetry}

\textcolor{black}{
Our FDR estimator in Eq.~\eqref{eq:fdr_math} relies
on approximate null symmetry of the mirror statistic.
The paired visual views
$f_v^{\pm}=\mathbf{v}\pm\tau_v Z_v$
share the same noise realization with opposite signs.
Consequently, the induced logit contrasts
$\Delta_t^{+}$ and $\Delta_t^{-}$ are coupled
and are not necessarily uncorrelated.
Conditioning on $Z_v$ fixes the perturbations
and does not justify an uncorrelatedness claim. 
}

\textcolor{black}{
Instead, we state the required condition directly
as approximate conditional sign-flip invariance
under the null. Define
\[
H_{0t}:
\text{no systematic visual evidence supports token } y_t,
\]
and let
\[
\mathcal{G}_t
=
\sigma(\mathbf{v},\mathbf{x},\mathbf{y}_{<t})
\]
denote the shared context consisting of the image,
textual input, and decoding history.
}
\textcolor{black}{
\begin{lemma}[Conditional null symmetry]
\label{lemma1}
Suppose that, under $H_{0t}$,
\begin{equation}
\left.
(\Delta_t^{+},\Delta_t^{-})
\right|H_{0t},\mathcal{G}_t
\;\overset{d}{\approx}\;
\left.
(\Delta_t^{+},-\Delta_t^{-})
\right|H_{0t},\mathcal{G}_t,
\label{eq:app_sign_flip}
\end{equation}
where $\overset{d}{\approx}$ denotes approximate
equality in distribution.
Then the mirror statistic satisfies
\begin{equation}
\left.\Delta_t\right|H_{0t},\mathcal{G}_t
\;\overset{d}{\approx}\;
\left.-\Delta_t\right|H_{0t},\mathcal{G}_t.
\label{eq:app_null_symmetry}
\end{equation}
Consequently, for every $s>0$,
\begin{equation}
\Pr(\Delta_t\ge s\mid H_{0t},\mathcal{G}_t)
\approx
\Pr(\Delta_t\le -s\mid H_{0t},\mathcal{G}_t).
\label{eq:app_tail_symmetry}
\end{equation}
Under exact conditional sign-flip invariance,
these relations hold exactly.
\end{lemma}
}

\textcolor{black}{
\begin{proof}
Define
\[
M(a,b)=|a+b|-|a-b|.
\]
The mirror functional is antisymmetric under
a sign flip of its second argument:
\[
M(a,-b)
=
|a-b|-|a+b|
=
-M(a,b).
\]
Applying the assumed conditional sign-flip
invariance gives
\begin{align}
\left.\Delta_t\right|H_{0t},\mathcal{G}_t
&=
\left.M(\Delta_t^{+},\Delta_t^{-})
\right|H_{0t},\mathcal{G}_t
\notag\\
&\overset{d}{\approx}
\left.M(\Delta_t^{+},-\Delta_t^{-})
\right|H_{0t},\mathcal{G}_t
\notag\\
&=
\left.-\Delta_t\right|H_{0t},\mathcal{G}_t.
\end{align}
This yields the corresponding approximate positive- and negative-tail equality. Exact sign-flip invariance gives exact equalities.
\end{proof}
}
\textcolor{black}{
The conditional sign-flip condition is a working
assumption for nonlinear LVLM responses; it is
not guaranteed by Gaussian input perturbations
alone. It replaces the unsupported claim that
the two logit contrasts are uncorrelated by
construction. The mirror statistic and empirical
thresholding procedure remain unchanged.
}



\textcolor{black}{
\begin{proof}[\textbf{Proof of Theorem~\ref{thm1}}]
The conditional sign-flip assumption in Theorem~\ref{thm1} is the exact-invariance case of Lemma~\ref{lemma1}. Applying that lemma to the mirror statistic defined in Eq.~\eqref{eq:mirror} yields, for every $s>0$,
\begin{equation}
    \Pr(\Delta_t\le-s\mid H_{0t},\mathcal{G}_t)
=
\Pr(\Delta_t\ge s\mid H_{0t},\mathcal{G}_t).
\end{equation}
which completes the proof.
\end{proof}
}




\textcolor{black}{The null symmetry in Theorem~\ref{thm1}, together with suitable tail-count concentration conditions, motivates using negative-tail counts to estimate false discoveries among positively selected outputs.}

To empirically assess whether this assumption holds in practice for LVLMs, we conduct a dedicated diagnostic analysis on \emph{negative POPE queries}, where the queried object is guaranteed to be absent from the image. In this setting, all object-level statistic correspond to null hypotheses, providing a controlled environment for evaluating symmetry.

\paragraph{Interpretation of the mirror statistic.}
The mirror statistic admits the equivalent expression
\begin{equation}
\Delta_t
=
2\,\operatorname{sign}(\Delta_t^{+}\Delta_t^{-})
\min\{|\Delta_t^{+}|,|\Delta_t^{-}|\}.
\end{equation}
Thus, $\Delta_t$ is positive when the two contrasts
have the same nonzero sign, negative when they have
opposite signs, and zero when either contrast is zero.
Its magnitude is twice the smaller absolute contrast.
A large positive value therefore indicates a strong,
directionally consistent response across the paired
views, rather than establishing visual grounding
by itself.

In particular, both
$\Delta_t^{+}>0,\Delta_t^{-}>0$ and
$\Delta_t^{+}<0,\Delta_t^{-}<0$
produce positive mirror statistics.
In the latter regime, the token's logit is higher
under both perturbed views than under the clean input.
The current threshold-based selection rule does not
explicitly distinguish these two regimes.
Consequently, interpreting selected outputs as
visually grounded depends on the validity of the
null-symmetry assumption and the enrichment of
grounded outputs in the positive tail.

To empirically assess whether this assumption holds in practice for LVLMs, we conduct a dedicated diagnostic analysis on \emph{negative POPE queries}, where the queried object is guaranteed to be absent from the image. In this setting, all object-level statistic correspond to null hypotheses, providing a controlled environment for evaluating symmetry.

\paragraph{Setup.}
We use the POPE random split on LLaVA-v1.5 and extract object-level mirror statistic $\Delta_t$ from 1500 negative (answer = ``no'') queries.
Each $\Delta_t$ is computed from mirrored visual features as described in Sec.~\ref{mirror}, using the logit-margin difference between the positive and negative mirror views. No filtering or thresholding is applied in this analysis.

\paragraph{Distributional symmetry.}
Figure~\ref{fig:mirror_symmetry} visualizes the empirical distribution of $\Delta_t$ and its sign-flipped counterpart $-\Delta_t$. The overlaid histograms show strong overlap, and the QQ plot (in Figure~\ref{fig:mirror_qq}) of $\Delta_t$ versus $-\Delta_t$ aligns closely with the identity line, indicating approximate symmetry across the full range of quantiles.

\begin{figure}[ht]
    \centering
    \includegraphics[width=0.8\linewidth]{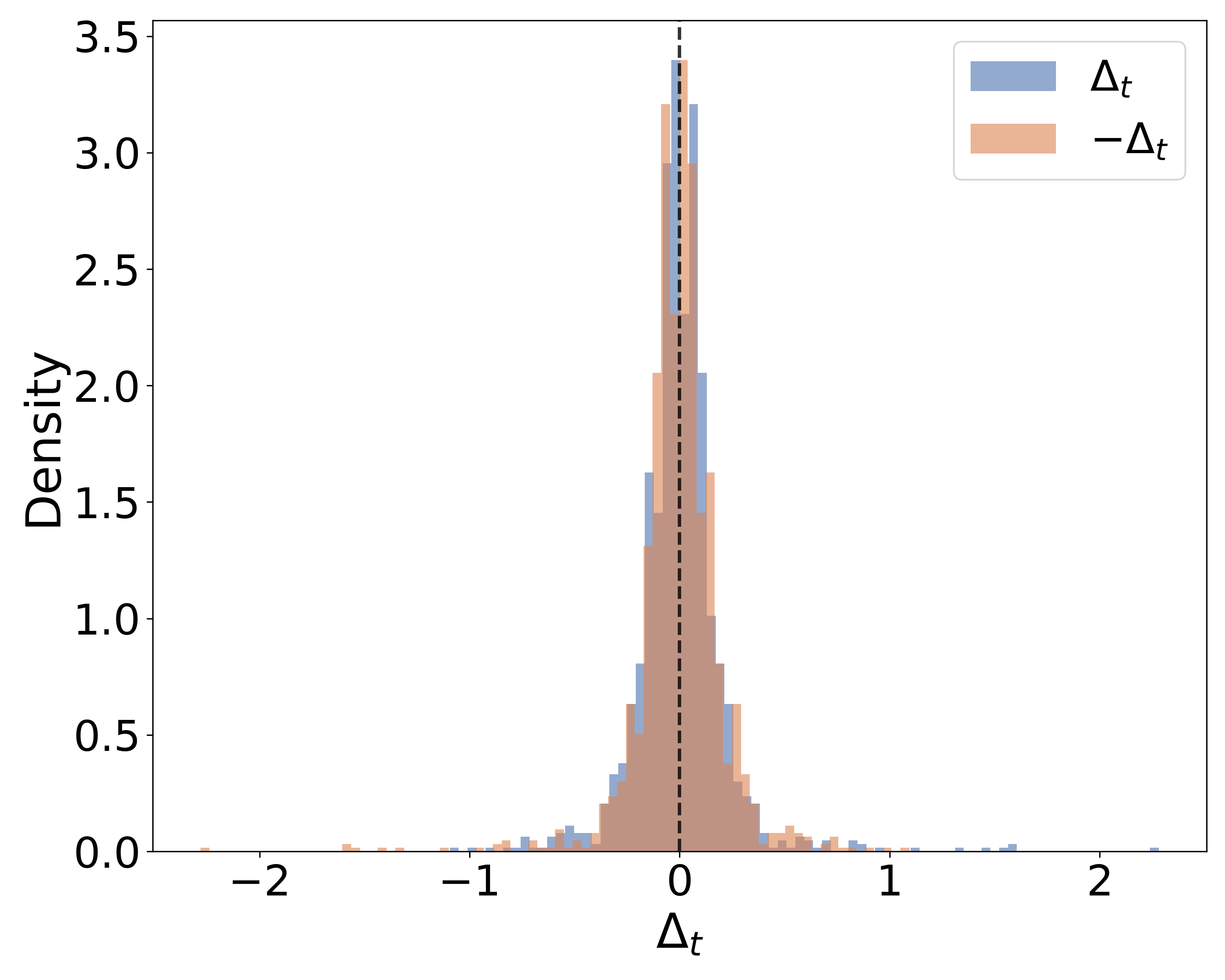}
    \caption{Empirical validation of mirror symmetry for mirror statistic $\Delta_t$. The figure shows overlaid histograms of $\Delta_t$ and its sign-flipped counterpart $-\Delta_t$ computed from negative POPE queries, where the queried object is guaranteed to be absent. The strong overlap between the two distributions and their near symmetry around zero indicate that mirror statistic for visually ungrounded objects are approximately symmetric, supporting their use as reference quantities for FDR estimation.}
    \label{fig:mirror_symmetry}
\end{figure}

\begin{figure}[ht]
    \centering
    \includegraphics[width=0.85\linewidth]{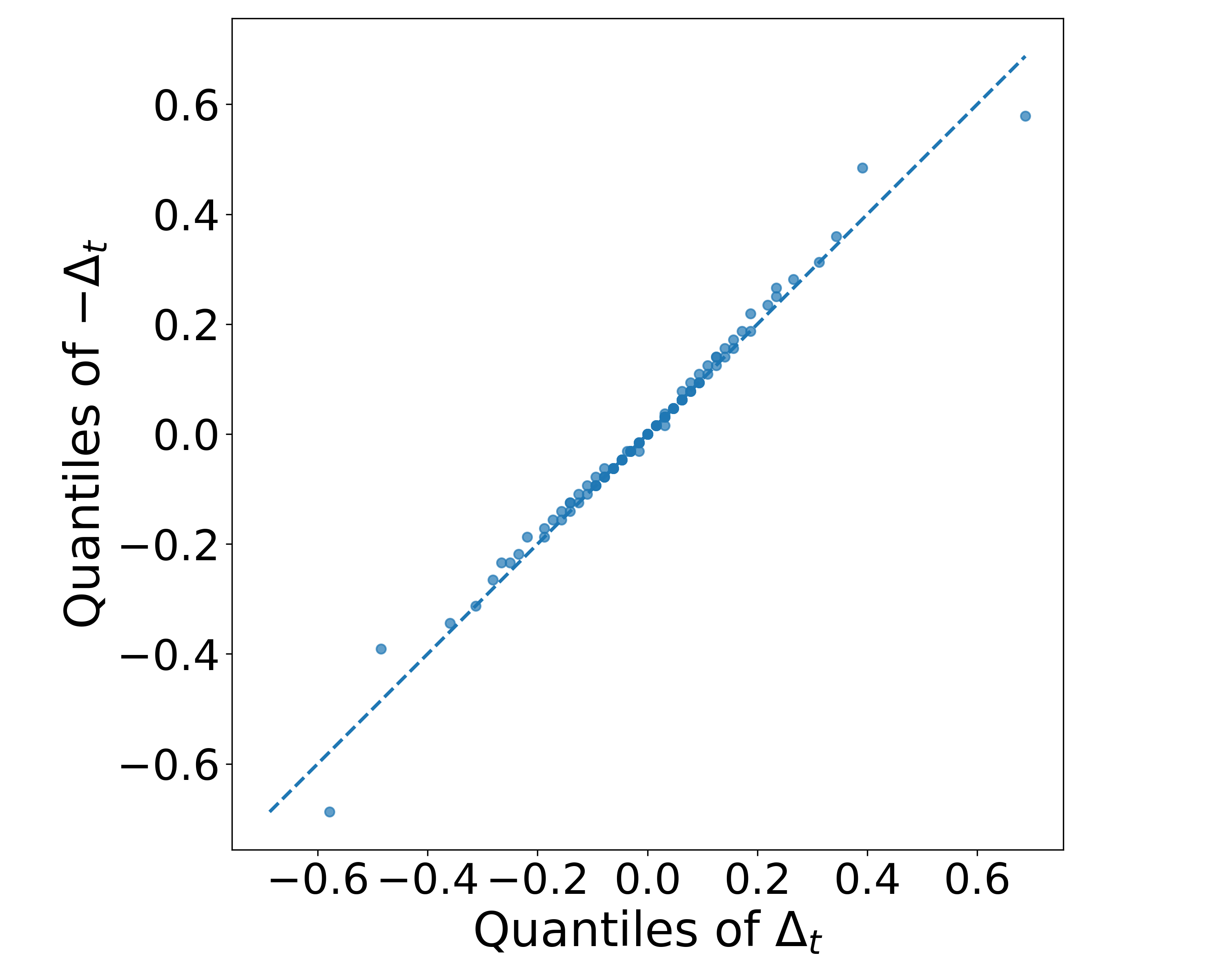}
    \caption{Quantile-Quantile (QQ) plot comparing the empirical quantiles of $\Delta_t$ and $-\Delta_t$ for negative POPE queries. The close alignment with the identity line across the full range of quantiles indicates approximate distributional symmetry of the mirror statistic, providing empirical support for the mirror-symmetry assumption used in Eq.~\eqref{eq:fdr_math}.}
    \label{fig:mirror_qq}
\end{figure}

\paragraph{Quantitative diagnostics.}
We further report numerical symmetry diagnostics. The empirical mean of $\Delta_t$ is close to zero ($\mathrm{mean}(\Delta_t) = 0.0012$), and the Kolmogorov-Smirnov test between $\Delta_t$ and $-\Delta_t$ yields a statistic of $0.0167$ with $p = 0.99$, failing to reject the null hypothesis that the two distributions are identical. These results indicate no detectable asymmetry at conventional significance levels.

\paragraph{Implications for FDR estimation.}
While LVLMs are highly nonlinear and do not strictly satisfy classical assumptions of mirror-based multiple testing, the above results suggest that, under symmetric feature perturbations, the induced mirror statistic for visually ungrounded objects are empirically well-approximated by a symmetric distribution. This empirical symmetry supports the use of Eq.~\eqref{eq:fdr_math} as a reliable plug-in estimator for the false discovery rate in our decoding-time selection procedure.

\subsection{Implementation Details: Object Sets and Token Span Construction}
\label{app:object_span}

This appendix provides a precise description of how object sets and token spans
are defined across tasks, and how object-level mirror statistic are constructed
for false discovery rate (FDR) control.

\paragraph{Object set definition.}
For a given image and prompt, $\method$ performs statistical testing at the object
level.
Let $\mathcal{S}$ denote the set of candidate objects associated with the input.
The definition of $\mathcal{S}$ is task-dependent.
For closed-set object existence benchmarks such as POPE and MME, $\mathcal{S}$
is directly given by the queried object categories provided by the benchmark
(e.g., \emph{fork}, \emph{bus}, \emph{zebra}).
For captioning evaluation with CHAIR, $\mathcal{S}$ consists of object categories
extracted from the generated caption using the standard CHAIR evaluation pipeline,
which matches noun phrases against the MSCOCO object vocabulary and its synonym
list.

\paragraph{Object mention identification.}
For each object $y_t \in \mathcal{S}$, we identify its occurrences in the generated
text by matching the canonical object name and its associated synonyms to the
generated sequence.
All matching is performed at the tokenizer level of the evaluated LVLM, ensuring
a deterministic and reproducible mapping.
Each matched occurrence corresponds to a contiguous span of subword tokens.

\paragraph{Token span construction.}
\textcolor{black}{
For captioning, let $\mathcal{T}_i$ denote the set of
token indices corresponding to all occurrences of
object category $i$ in the generated text.
If an object appears multiple times, $\mathcal{T}_i$
is the union of its corresponding token spans.
Token-level mirror statistics $\Delta_t$ are computed
for each associated token position $t$.
The constituent tokens are not treated as independent
object discoveries; instead, they are grouped into
a single object-level decision.
}

\textcolor{black}{
\paragraph{Object-level aggregation.}
For objects represented by multiple tokens, we summarize
the evidence using the least-supported token:
\begin{equation}
    \Delta_i = \min_{t\in\mathcal{T}_i}\Delta_t.
    \label{eq:object_aggregation}
\end{equation}
At the selection threshold $T_q$, object $i$ is retained
only if
\begin{equation}
    \Delta_i \ge T_q
    \quad\Longleftrightarrow\quad
    \Delta_t \ge T_q
    \;\; \text{for all } t\in\mathcal{T}_i.
    \label{eq:object_selection}
\end{equation}
This rule prevents an object from being retained solely
because one token has strong visual evidence while
another token in the associated span is weakly supported.
The grouped tokens contribute a single object-level
decision.
}

\textcolor{black}{
\paragraph{Scope of FDR control.}
Token-level FDR control does not, in general, automatically
imply object-level FDR control under an arbitrary
token-to-object mapping.
For POPE, the one-object-one-decision structure closely
aligns the tested decision with the evaluated object-level
hypothesis.
For multi-token objects and CHAIR-style caption evaluation,
a general provable mapping requires additional assumptions
on the parser, grouping function, and object-level null.
}

\textcolor{black}{
Accordingly, the guarantee is stated at the level at
which the testing procedure is applied, under its
required assumptions.
The object-level improvements observed in our experiments
provide empirical evidence of transfer to standard
object-hallucination metrics, rather than a universal
object-level FDR theorem for arbitrary structured outputs.
A hierarchical procedure that combines within-span
error control with FDR control across object groups
is a possible direction for future work.
}

\paragraph{Mirror statistic.}
For each token $y_t \in \mathcal{S}$, we define a token-level mirror statistic $\Delta_t$ based on the logit differences under the original and perturbed visual inputs, as described in Eq.~\eqref{eq:mirror}. This formulation operates directly at the token level and avoids the need for additional aggregation. As a result, the mirror-symmetry property required for valid FDR control is preserved. Object-level decisions can be derived from token-level statistic during evaluation by grouping tokens corresponding to the same object.


\paragraph{FDR and power computation.}
False discovery rate and power are computed at the image level over the object
set $\mathcal{S}$ using $\{\Delta_t\}_{y_t \in \mathcal{S}}$, rather than over
individual tokens.
This guarantees that the statistical testing procedure is directly aligned with
object-level hallucination evaluation protocols such as POPE, CHAIR, and MME.

\subsection{Statistical Dependence across Tokens and Objects}
\label{app:dependence}

Mirror statistics are computed separately for each
candidate decision, but are not assumed to be
mutually statistically independent.
Shared images, scene context, and decoding histories
can induce dependence among the statistics.

\paragraph{Autoregressive dependence.}
At decoding step $t$, the clean and mirror-view logits
are evaluated using the same fixed history
$\mathbf{y}_{<t}$.
The paired comparison is therefore made conditional
on a shared decoding history.
However, this does not imply independence of the
mirror statistics across decoding steps.

\paragraph{Null-tail symmetry and weak dependence.}
Let $\mathcal{S}_0$ index the null testing units.
For a candidate threshold $s>0$, define the
null-tail indicators
\begin{equation}
I_t^{+}(s)
=
\mathbf{1}\{\Delta_t\ge s\},
\qquad
I_t^{-}(s)
=
\mathbf{1}\{\Delta_t\le -s\}.
\label{eq:null_tail_indicators}
\end{equation}
The analysis requires approximate marginal symmetry,
\begin{equation}
\mathbb{E}[I_t^{+}(s)]
\approx
\mathbb{E}[I_t^{-}(s)],
\qquad t\in\mathcal{S}_0,
\label{eq:null_tail_balance}
\end{equation}
together with a weak-dependence condition such as
\begin{equation}
\sum_{t,u\in\mathcal{S}_0}
\left|
\operatorname{Cov}
\bigl(I_t^{\pm}(s),I_u^{\pm}(s)\bigr)
\right|
=
o\bigl(|\mathcal{S}_0|^2\bigr),
\label{eq:weak_dependence}
\end{equation}
uniformly over the candidate thresholds.
Here, the condition applies separately to the
positive and negative tail indicators as the
number of null testing units increases.

This condition allows correlations induced by
shared visual inputs and contextual information,
provided that their aggregate contribution
satisfies Eq.~\eqref{eq:weak_dependence}.
Local or block dependence can therefore be
compatible with the condition, whereas pervasive
dependence under which most null statistics move
together can violate it.

\paragraph{Dependence among object-level decisions.}
Hallucinated objects such as a dog, a car, and
a bicycle may be correlated through the same
scene context. Such correlation does not, by
itself, violate the weak-dependence condition.
The relevant requirement concerns concentration
of aggregate null-tail counts, rather than
independence of every pair of object statistics.

When testing is performed on object-level
statistics, the symmetry and dependence conditions
must hold for those statistics.
Thus, separate computation should not be
interpreted as statistical independence.
Within-image dependence diagnostics and end-to-end
empirical FDR calibration assess the plausibility
of these conditions in the evaluated settings.

\subsection{Implementation Details for Hallucination Evaluations} \label{app:implementation}
We evaluate the effectiveness of our methods on several state-of-the-art LVLMs, including LLaVA-v1.5 (7B and 13B) \cite{liu2023visual}, InstructBLIP (7B and 13B)~\citep{dai2023instructblip}, Qwen-VL (7B) \cite{bai2023Qwen}, LLaVA-OneVision-7B~\citep{li2024llava}, Qwen2.5-VL-7B~\citep{bai2025qwen25vl}, and InternVL3-8B~\citep{zhu2025internvl3}. We compare our methods with three state-of-the-art decoding methods, including VCD \cite{leng2024mitigating}, MARINE \cite{zhao2024mitigating}, and AGLA \cite{an2025mitigating}.
All evaluations were repeated over 3000 randomized trials. In each trial, we fix the model outputs and compute the evaluation metrics using random sampling procedures (e.g., POPE query sampling), rather than re-running full model inference. This allows efficient and stable estimation of mean and variance.
We followed the suggested settings in their respective papers and released codes to ensure fair comparison. For the POPE~\citep{li2023evaluating} datasets, $\alpha$ is set to 2 and $\beta$ is set to 0.5. For the proposed methods, the target level for FDR control is set to 0.1. For the CHAIR~\citep{rohrbach2018object}, $\alpha$ is set to 2 and $\beta$ is set to 0.5. For the MME~\citep{fu2025mme} dataset, we set $\alpha$ to 2 and $\beta$ to 0.5 for LLaVA-v1.5, and $\alpha$ and $\beta$ are set to 0.1 for InstructBLIP.In addition, the hyperparameter for VCD \citep{leng2024mitigating}, MARINE~\citep{zhao2024mitigating}, and AGLA~\citep{an2025mitigating} are reported in Table~\ref{tab:vcd_hyperparameters}, Table~\ref{tab:marine_hyperparameters}, and Table~\ref{tab:AGLA_hyperparameters} respectively. We strictly followed the original implementations and default hyperparameters described in their papers to reproduce each baseline's results. 

\begin{table}[ht]
\centering
\caption{VCD~\citep{leng2024mitigating} Hyperparameter Settings.}
\label{tab:vcd_hyperparameters}
\begin{tabular}{l|c}
\toprule
\textbf{Parameters} & \textbf{Value} \\
\hline
Amplification Factor $\alpha$
& 1
 \\
Adaptive Plausibility Threshold $\beta$
&
0.1
\\
Diffusion Noise Step
&
500
\\
\bottomrule
\end{tabular}
\end{table}

\begin{table}[ht]
\centering
\caption{MARINE~\citep{zhao2024mitigating} Hyperparameter Settings.}
\label{tab:marine_hyperparameters}
\begin{tabular}{l|c}
\toprule
\textbf{Parameters} & \textbf{Value} \\
\hline
Guidance Strength
& 0.7
 \\
Score Threshold for DERT
&
0.95
\\
Detect Threshold for RAM++
&
0.68
\\
\bottomrule
\end{tabular}
\end{table}

\begin{table}[ht]
\centering
\caption{AGLA~\citep{an2025mitigating} Hyperparameter Settings.}
\label{tab:AGLA_hyperparameters}
\begin{tabular}{l|c}
\toprule
\textbf{Parameters} & \textbf{Value} \\
\hline
Weighting Factor $\alpha$
& 2
 \\
Adaptive Plausibility Constraint Factor $\beta$
&
0.5
\\
\bottomrule
\end{tabular}
\end{table}

Key factors that potentially affect the hallucination evaluation outcomes, including the evaluation dataset and prompt template, LVLM’s sampling strategy and batched generation techniques, data splitting $\tau$, and FDR target level $q$. The hyperparameter settings for $\method$ and overall experiment settings are shown in Table~\ref{tab:CORAL_hyperparameters} and Table~\ref{tab:batch_size}.

\begin{table}[ht]
\centering
\caption{$\method$ Hyperparameter Settings. The settings are fixed depending on the question-answer tasks.}
\label{tab:CORAL_hyperparameters}
\begin{tabular}{l|c}
\toprule
\textbf{Parameters} & \textbf{Value} \\
\hline
Data Splitting Factor $\tau$
& 0.1
 \\
FDR Thresholding $q$ (false positive level in an image)
&
0.1
\\
\bottomrule
\end{tabular}
\end{table}

\begin{table}[t]
\centering
\caption{Batch size settings for LVLM generation across different models. Unless otherwise noted, the batch size is fixed for each model throughout all experiments. To improve evaluation efficiency, we adopt batched generation. When the LVLM does not explicitly specify a padding strategy for inference, we apply left padding to avoid potential negative effects of batched generation.}
\label{tab:batch_size}
\begin{tabular}{lc}
\toprule
\textbf{Model} & \textbf{Batch Size} \\
\midrule
LLaVA-v1.5 & 4 \\
Qwen-VL & 16 \\
InstructBLIP & 16 \\
LLaVA-OneVision-7B & 4 \\
Qwen2.5-VL-7B & 16 \\
InternVL3-8B & 8 \\
\bottomrule
\end{tabular}
\end{table}

\paragraph{Experiments for POPE Evaluations} POPE is a flexible approach to evaluating hallucinations in LVLMs, which formulates a binary classification task by prompting LVLMs with questions such as “Is there a keyboard in this image?” to answer “yes” or “no”. Following~\citep{leng2024mitigating}, the POPE benchmark aggregates data from three distinct sources: MSCOCO~\citep{lin2014microsoft}, A-OKVQA~\citep{schwenk2022okvqa}, and GQA~\citep{hudson2019gqa}. It involves 500 images from each dataset under each sampling setting and formulates 6 questions per image, culminating in a total of 27,000 query answer pairs from the development sets of these datasets. We reported the results on the MSCOCO dataset in Table~\ref{tab:pope_coco1_random}, results on the A-OKVQA dataset in Table~\ref{tab:pope_vaq}, and results on the A-OKVQA dataset in Table~\ref{tab:pope_gqa}.

\paragraph{Experiments for CHAIR Evaluations} 

\textit{Caption Hallucination Assessment with Image Relevance} (CHAIR) \cite{rohrbach2018object} quantifies object hallucinations in {image captions} by comparing generated objects to ground-truth ones. 
We adopt the same prompt “Generate a short caption of the image.” as utilized by ~\citeauthor{li2023evaluating}~(\citeyear{li2023evaluating}). The maximum token length is 64, and the sampling approach with random seed of 242. For the calculation of CHAIR metrics, we referenced the 80 object categories annotated in the MSCOCO dataset, following~\citep{rohrbach2018object}. 
\begin{align*}
    \text{CHAIR}_I = \frac{|\{\text{hallucinated objects}\}|}{|\{\text{all mentioned objects}\}|}, \quad 
    \text{CHAIR}_S = \frac{|\{\text{captions with hallucinated objects}\}|}{|\{\text{all captions}\}|}
\end{align*}
Besides, we employed the synonym list from~\citep{lu2018neural} to align synonymous words in the generated text with MSCOCO object categories. And we reported the result in Table~\ref{tab:chair}.  Following previous work~\citep{zhao2024mitigating}, we randomly select 500 images from MSCOCO~\citep{lin2014microsoft} and use $\text{CHAIR}_I$ and $\text{CHAIR}_S$. Additionally, we incorporate the overall power (Eq.~\eqref{eq:power})  to evaluate whether the descriptions accurately include the necessary visual content from the image. $\method$ consistently reduces both sentence-level and instance-level hallucination rates while maintaining strong power across all models.

\begin{table}[ht]
\centering
\caption{Evaluation with CHAIR score across multiple LVLM architectures.
Lower $C_S$ and $C_I$ indicate fewer hallucinated objects, while higher power indicates better retention of visually grounded objects.
\textbf{Bold} indicates the best result. }
\label{tab:chair}
\setlength{\tabcolsep}{4pt}
\renewcommand{\arraystretch}{1.15}
{\large
\resizebox{\columnwidth}{!}{%
\begin{tabular}{lcccccccccccccccccc}
\toprule
\multirow{2}{*}{Method}
& \multicolumn{3}{c}{LLaVA-v1.5}
& \multicolumn{3}{c}{Qwen-VL}
& \multicolumn{3}{c}{InstructBLIP}
& \multicolumn{3}{c}{LLaVA-OneVision-7B}
& \multicolumn{3}{c}{Qwen2.5-VL-7B}
& \multicolumn{3}{c}{InternVL3-8B} \\

\cmidrule(lr){2-4}\cmidrule(lr){5-7}\cmidrule(lr){8-10}
\cmidrule(lr){11-13}\cmidrule(lr){14-16}\cmidrule(lr){17-19}

& $C_S\downarrow$ & $C_I\downarrow$ & Power$\uparrow$
& $C_S\downarrow$ & $C_I\downarrow$ & Power$\uparrow$
& $C_S\downarrow$ & $C_I\downarrow$ & Power$\uparrow$
& $C_S\downarrow$ & $C_I\downarrow$ & Power$\uparrow$
& $C_S\downarrow$ & $C_I\downarrow$ & Power$\uparrow$
& $C_S\downarrow$ & $C_I\downarrow$ & Power$\uparrow$ \\

\midrule

Regular
& 9.2 & 5.1 & 94.8
& 9.5 & 19.2 & 80.7
& 8.7 & 4.9 & 95.2
& 8.8 & 4.6 & 95.4
& 8.2 & 18.1 & 81.9
& 5.0 & 3.2 & 96.8 \\

VCD (\citeyear{leng2024mitigating})
& 7.8 & 4.5 & 95.3
& 7.4 & 18.5 & 81.3
& 7.3 & 4.1 & 95.9
& 7.3 & 4.1 & 95.9
& 6.8 & 17.4 & 82.6
& 2.4 & 1.5 & 98.5 \\

MARINE (\citeyear{zhao2024mitigating})
& 6.9 & 3.8 & 96.1
& 6.3 & 14.5 & 84.7
& 6.2 & 3.0 & 97.0
& 6.2 & 3.0 & 97.0
& 5.9 & 13.8 & 86.2
& 2.2 & \textbf{1.3} & \textbf{98.7} \\

AGLA (\citeyear{an2025mitigating})
& 7.5 & 4.2 & 95.8
& 6.1 & 12.4 & 87.2
& 7.0 & 3.8 & 96.2
& 7.0 & 3.8 & 96.2
& 5.6 & 11.2 & 88.8
& 2.3 & 1.6 & 98.4 \\

\rowcolor{blue!8}
$\method$ (Ours)
& \textbf{5.8} & \textbf{3.1} & \textbf{96.9}
& \textbf{4.5} & \textbf{11.0} & \textbf{88.9}
& \textbf{5.0} & \textbf{2.6} & \textbf{97.4}
& \textbf{5.0} & \textbf{2.6} & \textbf{97.4}
& \textbf{3.8} & \textbf{10.8} & \textbf{89.2}
& \textbf{1.8} & \textbf{1.3} & \textbf{98.7} \\

\bottomrule
\end{tabular}
}
}
\end{table}

\paragraph{Experiments for MME Evaluations} Similar to the POPE dataset, the MME dataset~\citep{fu2025mme} contains only two types of answers (i.e., Yes or No). For hallucination-related tasks, MME likewise relies on \textit{Yes-or-No} judgments, yielding instance-level correctness outcomes that can be aggregated for performance comparison.  Following the setting in their original paper, we use the sum of accuracy and accuracy+ as the final score, where accuracy is calculated based on each question, and accuracy+ is calculated based on each image where both of the two questions need to be answered correctly. So accuracy+ is a stricter measurement that can better reflect the comprehensive understanding degree of the model. We reported our results in Table~\ref{tab:mme}.

\begin{table*}[h]
\centering
\caption{Evaluation with MME score across multiple LVLM architectures on MSCOCO with 3000 random replications.
We report the mean $\pm$ standard deviation. \textbf{Bold} indicates the best result and \uline{underline} indicates the second-best result.}
\label{tab:mme}
\setlength{\tabcolsep}{4pt}
\renewcommand{\arraystretch}{1.15}
\resizebox{\textwidth}{!}{%
\begin{tabular}{llccccc}
\toprule
\multirow{2}{*}{Model} & \multirow{2}{*}{Method}
& \multicolumn{3}{c}{Attribute}
& \multicolumn{1}{c}{Relation}
& \multirow{2}{*}{Total $\uparrow$} \\
\cmidrule(lr){3-5} \cmidrule(lr){6-6}
&
& Existence $\uparrow$
& Count $\uparrow$
& Color $\uparrow$
& Position $\uparrow$
& \\
\midrule

\multirow{5}{*}{LLaVA-v1.5}
& Regular
& 175.67 ($\pm$ 7.51) & 124.67 ($\pm$ 19.59) & 151.00 ($\pm$ 10.45) & 114.00 ($\pm$ 9.32) & 565.33 ($\pm$ 33.92) \\
& VCD (\citeyear{leng2024mitigating})
& 184.66 ($\pm$ 6.81) & 138.33 ($\pm$ 15.68) & 153.00 ($\pm$ 7.58) & 128.67 ($\pm$ 7.21) & 604.66 ($\pm$ 18.76) \\
& MARINE (\citeyear{zhao2024mitigating})
& 190.53 ($\pm$ 7.26) & \uline{154.43 ($\pm$ 16.01)} & 166.34 ($\pm$ 6.96) & \uline{130.28 ($\pm$ 8.01)} & 641.58 ($\pm$ 17.47) \\
& AGLA (\citeyear{an2025mitigating})
& \textbf{195.00 ($\pm$ 7.32)} & 153.89 ($\pm$ 16.32) & \textbf{167.67 ($\pm$ 6.42)} & 129.44 ($\pm$ 7.81) & \uline{646.00 ($\pm$ 15.96)} \\
\rowcolor{blue!8}
& $\method$ (Ours)
& \uline{194.43 ($\pm$ 8.38)} & \textbf{157.41 ($\pm$ 15.11)} & \uline{167.34 ($\pm$ 7.33)} & \textbf{131.19 ($\pm$ 7.58)} & \textbf{650.37 ($\pm$ 18.12)} \\
\midrule

\multirow{5}{*}{Qwen-VL}
& Regular
& 155.00 ($\pm$ 3.54) & 127.67 ($\pm$ 13.36) & 173.00 ($\pm$ 9.75) & 131.67 ($\pm$ 7.73) & 587.33 ($\pm$ 31.06) \\
& VCD (\citeyear{leng2024mitigating})
& 156.00 ($\pm$ 6.25) & 131.00 ($\pm$ 6.19) & 181.67 ($\pm$ 5.14) & 128.00 ($\pm$ 3.61) & 596.67 ($\pm$ 11.61) \\
& MARINE (\citeyear{zhao2024mitigating})
& 164.20 ($\pm$ 6.73) & \textbf{136.55 ($\pm$ 6.24)} & 185.79 ($\pm$ 6.21) & 132.37 ($\pm$ 7.76) & 618.91 ($\pm$ 11.88) \\
& AGLA (\citeyear{an2025mitigating})
& \uline{165.78 ($\pm$ 5.28)} & 134.18 ($\pm$ 7.14) & \uline{187.12 ($\pm$ 5.33)} & \uline{133.17 ($\pm$ 7.77)} & \uline{620.25 ($\pm$ 13.42)} \\
\rowcolor{blue!8}
& $\method$ (Ours)
& \textbf{166.00 ($\pm$ 6.58)} & \uline{135.21 ($\pm$ 7.09)} & \textbf{188.22 ($\pm$ 7.76)} & \textbf{134.00 ($\pm$ 5.49)} & \textbf{623.43 ($\pm$ 12.67)} \\
\midrule

\multirow{5}{*}{InstructBLIP}
& Regular
& 141.00 ($\pm$ 13.97) & \textbf{75.33 ($\pm$ 14.16)} & 97.33 ($\pm$ 16.94) & \textbf{66.67 ($\pm$ 3.91)} & 380.33 ($\pm$ 40.20) \\
& VCD (\citeyear{leng2024mitigating})
& 170.00 ($\pm$ 11.55) & 61.67 ($\pm$ 8.47) & 114.44 ($\pm$ 11.27) & 57.22 ($\pm$ 6.73) & 403.33 ($\pm$ 13.36) \\
& MARINE (\citeyear{zhao2024mitigating})
& \textbf{189.43 ($\pm$ 11.88)} & 63.77 ($\pm$ 7.80) & \uline{118.67 ($\pm$ 12.15)} & 66.00 ($\pm$ 7.25) & \textbf{437.87 ($\pm$ 13.42)} \\
& AGLA (\citeyear{an2025mitigating})
& 180.00 ($\pm$ 12.11) & 63.33 ($\pm$ 7.39) & \textbf{119.44 ($\pm$ 12.11)} & 65.56 ($\pm$ 8.48) & 428.33 ($\pm$ 12.81) \\
\rowcolor{blue!8}
& $\method$ (Ours)
& \uline{182.03 ($\pm$ 12.08)} & \uline{64.17 ($\pm$ 8.53)} & 118.25 ($\pm$ 12.76) & \uline{66.43 ($\pm$ 7.83)} & \uline{431.18 ($\pm$ 12.30)} \\

\bottomrule
\end{tabular}
}
\end{table*}

\paragraph{Experiments for Latency Analysis} \label{app:latency}
Experiment setting for latency analysis. We compared our method with existing baselines in terms of the trade-off between inference cost and the effectiveness of reducing object hallucinations, as shown in Figure~\ref{fig:latency}. For decoding methods such as VCD, AGLA, MARINE and our method, we measured the latency of LLaVA-v1.5 generating captions directly. We prompted the models with “Generate a short caption of the image.” on 500 MSCOCO images with a batch size of 1 and a maximum token length of 64, without any stopping criteria, using a single RXT 6000 GPU. Then latency was calculated as the ratio of the number of output tokens and encoding and generation time.

\begin{figure}[t]
    \centering
    \includegraphics[width=0.7\linewidth]{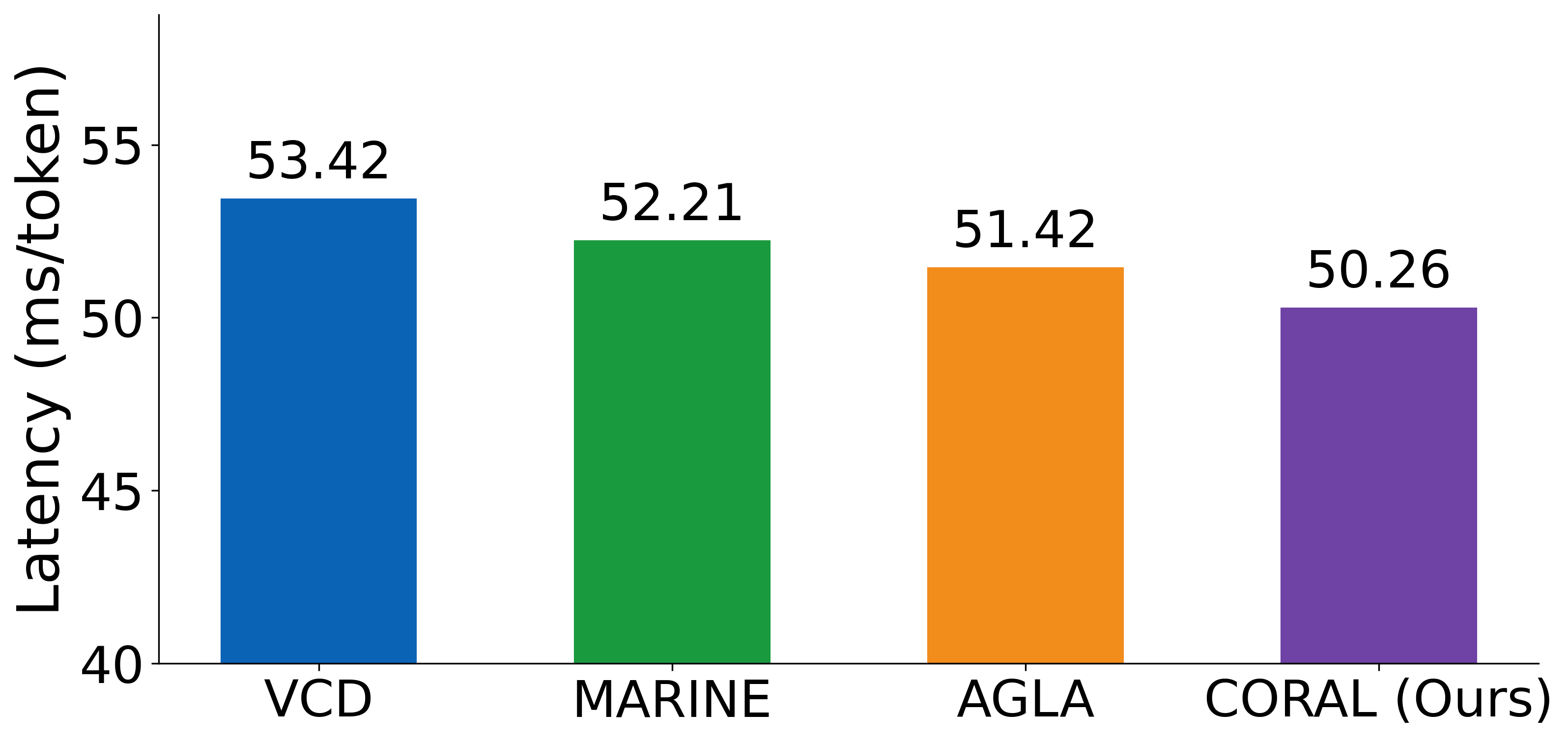}
    \caption{Inference latency comparisons. All measurements are conducted on a single NVIDIA RTX 6000 GPU with batch size 1, using identical input images and prompts.}
    \label{fig:latency}
\end{figure}

\subsection{Additional Experiment Results on FDR and Power}
Additionally, we report additional experimental results on false discovery rate (FDR) and power under a fixed target level $q = 0.1$. Throughout all experiments, we apply the same setting across three datasets. The target level $q = 0.1$ corresponds to controlling the proportion of false positives per image, such that the expected fraction of falsely retained hallucinated objects does not exceed $0.1$ for each image. This setting is used consistently for all datasets to ensure a fair and comparable evaluation. We reported the result for MSCOCO~\citep{lin2014microsoft} in Table~\ref{tab:mlt_coco_old}. Additional results on A-OKVQA~\citep{schwenk2022okvqa} and GQA~\citep{hudson2019gqa} under this setting are provided in Table~\ref{tab:mlt_vaq} and Table~\ref{tab:mlt_gqa}.

\begin{table*}[ht]
\centering
\caption{Evaluation of overall FDR control and overall power across multiple LVLM architectures on MSCOCO with 3000 random replications. Overall FDR and power denote averages of per-image FDR and power across images. \textbf{Bold} indicates the best result and \uline{underline} indicates the second-best result in each column.
}
\label{tab:mlt_coco_old}
\setlength{\tabcolsep}{3pt}
\renewcommand{\arraystretch}{1.15}
\resizebox{\textwidth}{!}{%
\begin{tabular}{lcccccccc}
\toprule
\multirow{2}{*}{Method}
& \multicolumn{2}{c}{LLaVA-v1.5}
& \multicolumn{2}{c}{Qwen-VL}
& \multicolumn{2}{c}{InstructBLIP}
& \multicolumn{2}{c}{{Average}} \\
\cmidrule(lr){2-3}\cmidrule(lr){4-5}\cmidrule(lr){6-7}\cmidrule(lr){8-9}
& Overall FDR$\downarrow$ & Overall Power$\uparrow$
& Overall FDR$\downarrow$ & Overall Power$\uparrow$
& Overall FDR$\downarrow$ & Overall Power$\uparrow$
& Overall FDR$\downarrow$ & Overall Power$\uparrow$ \\
\midrule

\multicolumn{9}{l}{\textit{Random}}\\
\midrule
Regular
& 0.0985 ($\pm$ 0.0144) & 77.32 ($\pm$ 1.41)
& 0.0913 ($\pm$ 0.0122) & 76.08 ($\pm$ 1.23)
& 0.0957 ($\pm$ 0.0102) & 74.35 ($\pm$ 1.63)
& 0.0952 ($\pm$ 0.0122) & 75.92 ($\pm$ 1.42) 
\\
VCD (\citeyear{leng2024mitigating})
& 0.0965 ($\pm$ 0.0200) & 82.50 ($\pm$ 1.13)
& 0.0898 ($\pm$ 0.0111) & 78.10 ($\pm$ 1.12)
& 0.0941 ($\pm$ 0.0119) & 76.80 ($\pm$ 1.22)
& 0.0935 ($\pm$ 0.0143) & 79.13 ($\pm$ 1.16) \\
MARINE (\citeyear{zhao2024mitigating})
& 0.0935 ($\pm$ 0.0112) & 89.20 ($\pm$ 1.07)
& 0.0852 ($\pm$ 0.0119) & 79.05 ($\pm$ 1.13)
& 0.0944 ($\pm$ 0.0114) & 77.14 ($\pm$ 1.35)
& 0.0910 ($\pm$ 0.0115) & 81.80 ($\pm$ 1.18) \\
AGLA (\citeyear{an2025mitigating})
& \uline{0.0918 ($\pm$ 0.0107)} & \uline{92.80 ($\pm$ 1.15)}
& \uline{0.0818 ($\pm$ 0.0132)} & \uline{82.67 ($\pm$ 1.10)}
& \uline{0.0921 ($\pm$ 0.0122)} & \uline{78.92 ($\pm$ 1.21)}
& \uline{0.0886 ($\pm$ 0.0120)} & \uline{84.80 ($\pm$ 1.15)} \\
\rowcolor{blue!8}
\method (Ours)
& \textbf{0.0910 ($\pm$ 0.0101)} & \textbf{97.62 ($\pm$ 1.02)}
& \textbf{0.0755 ($\pm$ 0.0113)} & \textbf{84.90 ($\pm$ 1.09)}
& \textbf{0.0836 ($\pm$ 0.0100)} & \textbf{84.00 ($\pm$ 1.34)}
& \textbf{0.0834 ($\pm$ 0.0105)} & \textbf{88.84 ($\pm$ 1.15)} \\
\bottomrule
\end{tabular}
}
\vspace{-0.1in}
\end{table*}

\begin{table*}[ht]
\centering
\caption{Evaluation on overall FDR control and overall power across multiple LVLM architectures on \textbf{A-OKVQA} with 3000 random replications.
\textbf{Bold} indicates the best result and \uline{underline} indicates the second-best result in each column.}
\label{tab:mlt_vaq}
\setlength{\tabcolsep}{4pt}
\renewcommand{\arraystretch}{1.15}
\resizebox{\textwidth}{!}{%
\begin{tabular}{lcccccccc}
\toprule
Method
& \multicolumn{2}{c}{LLaVA-v1.5}
& \multicolumn{2}{c}{Qwen-VL}
& \multicolumn{2}{c}{InstructBLIP}
& \multicolumn{2}{c}{Average} \\
\cmidrule(lr){2-3} \cmidrule(lr){4-5} \cmidrule(lr){6-7} \cmidrule(lr){8-9}
& FDR$\downarrow$ & Power$\uparrow$
& FDR$\downarrow$ & Power$\uparrow$
& FDR$\downarrow$ & Power$\uparrow$
& FDR$\downarrow$ & Power$\uparrow$ \\
\midrule

\multicolumn{9}{l}{\textit{Random}} \\
\midrule
Regular
& 0.0957 ($\pm$ 0.0131) & 77.57 ($\pm$ 1.28)
& 0.0953 ($\pm$ 0.0133) & 79.53 ($\pm$ 1.52)
& 0.0987 ($\pm$ 0.0107) & 75.34 ($\pm$ 1.26)
& 0.0966 ($\pm$ 0.0124) & 77.48 ($\pm$ 1.35)
\\
VCD (\citeyear{leng2024mitigating})
& 0.0930 ($\pm$ 0.0100) & 84.10 ($\pm$ 1.22)
& 0.0892 ($\pm$ 0.0121) & 82.30 ($\pm$ 1.13)
& 0.0967 ($\pm$ 0.0112) & 80.40 ($\pm$ 1.15)
& 0.0930 ($\pm$ 0.0111) & 82.27 ($\pm$ 1.17)\\
MARINE (\citeyear{zhao2024mitigating})
& 0.0920 ($\pm$ 0.0101) & 85.02 ($\pm$ 1.35)
& 0.0835 ($\pm$ 0.0202) & 86.64 ($\pm$ 1.05)
& \uline{0.0920 ($\pm$ 0.0100)} & 84.93 ($\pm$ 1.08)
& 0.0892 ($\pm$ 0.0134) & 85.53 ($\pm$ 1.16)\\
AGLA (\citeyear{an2025mitigating})
& \uline{0.0892 ($\pm$ 0.0112)} & \uline{86.42 ($\pm$ 1.42)}
& \uline{0.0832 ($\pm$ 0.0158)} & \uline{87.01 ($\pm$ 1.11)}
& {0.0923 ($\pm$ 0.0123)} & \uline{87.43 ($\pm$ 1.11)}
& \uline{0.0882 ($\pm$ 0.0131)} & \uline{86.95 ($\pm$ 1.21)} \\
\rowcolor{blue!8}
$\method$ (Ours)
& \textbf{0.0781 ($\pm$ 0.0201)} & \textbf{92.40 ($\pm$ 1.22)}
& \textbf{0.0745 ($\pm$ 0.0133)} & \textbf{90.85 ($\pm$ 1.08)}
& \textbf{0.0831 ($\pm$ 0.0101)} & \textbf{88.95 ($\pm$ 1.14)}
& \textbf{0.0786 ($\pm$ 0.0145)} & \textbf{90.73 ($\pm$ 1.15)} \\
\bottomrule
\end{tabular}
}
\setlength{\tabcolsep}{4pt}
\renewcommand{\arraystretch}{1.15}
\resizebox{\textwidth}{!}{%
\begin{tabular}{lcccccccc}
\toprule
Method
& \multicolumn{2}{c}{LLaVA-OneVision-7B}
& \multicolumn{2}{c}{Qwen2.5-VL-7B}
& \multicolumn{2}{c}{InternVL3-8B}
& \multicolumn{2}{c}{Average} \\
\cmidrule(lr){2-3} \cmidrule(lr){4-5} \cmidrule(lr){6-7} \cmidrule(lr){8-9}
& FDR$\downarrow$ & Power$\uparrow$
& FDR$\downarrow$ & Power$\uparrow$
& FDR$\downarrow$ & Power$\uparrow$
& FDR$\downarrow$ & Power$\uparrow$ \\
\midrule
\multicolumn{9}{l}{\textit{Random}} \\
\midrule
Regular
& 0.0972 ($\pm$ 0.0110) & 70.46 ($\pm$ 1.54)
& 0.0965 ($\pm$ 0.0102) & 71.43 ($\pm$ 1.29)
& 0.0980 ($\pm$ 0.0118) & 73.28 ($\pm$ 1.25)
& 0.0972 ($\pm$ 0.0110) & 71.72 ($\pm$ 1.36)
\\
VCD (\citeyear{leng2024mitigating})
& 0.0942 ($\pm$ 0.0115) & 75.37 ($\pm$ 1.26)
& 0.0901 ($\pm$ 0.0110) & 77.29 ($\pm$ 1.16)
& 0.0921 ($\pm$ 0.0114) & 76.25 ($\pm$ 1.25)
& 0.0921 ($\pm$ 0.0113) & 76.30 ($\pm$ 1.22)\\
MARINE (\citeyear{zhao2024mitigating})
& 0.0918 ($\pm$ 0.0116) & 77.01 ($\pm$ 1.66)
& 0.0900 ($\pm$ 0.0155) & 78.03 ($\pm$ 1.13)
& 0.0891 ($\pm$ 0.0145) & 80.35 ($\pm$ 1.56)
& 0.0903 ($\pm$ 0.0139) & 78.46 ($\pm$ 1.45)\\
AGLA (\citeyear{an2025mitigating})
& \uline{0.0874 ($\pm$ 0.0203)} & \uline{84.22 ($\pm$ 1.77)}
& \uline{0.0868 ($\pm$ 0.0117)} & \uline{85.38 ($\pm$ 1.10)}
& \uline{0.0877 ($\pm$ 0.0201)} & \uline{83.89 ($\pm$ 1.26)}
& \uline{0.0873 ($\pm$ 0.0174)} & \uline{84.50 ($\pm$ 1.38)} \\
\rowcolor{blue!8}
$\method$ (Ours)
& \textbf{0.0772 ($\pm$ 0.0189)} & \textbf{89.20 ($\pm$ 1.65)}
& \textbf{0.0714 ($\pm$ 0.0200)} & \textbf{92.25 ($\pm$ 1.26)}
& \textbf{0.0733 ($\pm$ 0.0177)} & \textbf{90.76 ($\pm$ 1.41)}
& \textbf{0.0740 ($\pm$ 0.0189)} & \textbf{90.74 ($\pm$ 1.44)} \\
\bottomrule
\end{tabular}
}
\end{table*}

\begin{table*}[ht]
\centering
\caption{Evaluation on overall FDR control and overall power across multiple LVLM architectures on \textbf{GQA} with 3000 random replications.
\textbf{Bold} indicates the best result and \uline{underline} indicates the second-best result in each column.}
\label{tab:mlt_gqa}
\setlength{\tabcolsep}{4pt}
\renewcommand{\arraystretch}{1.15}
\resizebox{\textwidth}{!}{%
\begin{tabular}{lcccccccc}
\toprule
Method
& \multicolumn{2}{c}{LLaVA-v1.5}
& \multicolumn{2}{c}{Qwen-VL}
& \multicolumn{2}{c}{InstructBLIP}
& \multicolumn{2}{c}{Average} \\
\cmidrule(lr){2-3} \cmidrule(lr){4-5} \cmidrule(lr){6-7} \cmidrule(lr){8-9}
& FDR$\downarrow$ & Power$\uparrow$
& FDR$\downarrow$ & Power$\uparrow$
& FDR$\downarrow$ & Power$\uparrow$
& FDR$\downarrow$ & Power$\uparrow$ \\
\midrule

\multicolumn{9}{l}{\textit{Random}} \\
\midrule
Regular
& 0.0933 ($\pm$ 0.0132) & 76.55 ($\pm$ 1.32)
& 0.0914 ($\pm$ 0.0188) & 81.66 ($\pm$ 1.01)
& 0.0889 ($\pm$ 0.0188) & 83.88 ($\pm$ 1.52)
& 0.0912 ($\pm$ 0.0169) & 80.69 ($\pm$ 1.28)
\\
VCD (\citeyear{leng2024mitigating})
& 0.0924 ($\pm$ 0.0098) & 80.92 ($\pm$ 0.68)
& 0.0924 ($\pm$ 0.0090) & 80.24 ($\pm$ 1.00)
& 0.0872 ($\pm$ 0.0145) & 85.47 ($\pm$ 0.62)
& 0.0907 ($\pm$ 0.0111) & 82.21 ($\pm$ 0.77)\\
MARINE (\citeyear{zhao2024mitigating})
& 0.0894 ($\pm$ 0.0128) & 86.92 ($\pm$ 0.40)
& 0.0902 ($\pm$ 0.0101) & 84.22 ($\pm$ 0.93)
& 0.0837 ($\pm$ 0.0144) & 87.39 ($\pm$ 0.55)
& 0.0878 ($\pm$ 0.0124) & 86.18 ($\pm$ 0.63)\\
AGLA (\citeyear{an2025mitigating})
& \uline{0.0802 ($\pm$ 0.0164)} & \uline{90.42 ($\pm$ 1.02)}
& \uline{0.0872 ($\pm$ 0.0100)} & \uline{85.11 ($\pm$ 1.09)}
& \uline{0.0855 ($\pm$ 0.0148)} & \uline{89.43 ($\pm$ 0.41)}
& \uline{0.0843 ($\pm$ 0.0137)} & \uline{88.99 ($\pm$ 0.84)} \\
\rowcolor{blue!8}
$\method$ (Ours)
& \textbf{0.0719 ($\pm$ 0.0223)} & \textbf{92.05 ($\pm$ 0.32)}
& \textbf{0.0825 ($\pm$ 0.0121)} & \textbf{87.15 ($\pm$ 0.44)}
& \textbf{0.0813 ($\pm$ 0.0170)} & \textbf{90.15 ($\pm$ 0.67)}
& \textbf{0.0786 ($\pm$ 0.0171)} & \textbf{89.78 ($\pm$ 0.48)} \\
\bottomrule
\end{tabular}
}
\setlength{\tabcolsep}{4pt}
\renewcommand{\arraystretch}{1.15}
\resizebox{\textwidth}{!}{%
\begin{tabular}{lcccccccc}
\toprule
Method
& \multicolumn{2}{c}{LLaVA-OneVision-7B}
& \multicolumn{2}{c}{Qwen2.5-VL-7B}
& \multicolumn{2}{c}{InternVL3-8B}
& \multicolumn{2}{c}{Average} \\
\cmidrule(lr){2-3} \cmidrule(lr){4-5} \cmidrule(lr){6-7} \cmidrule(lr){8-9}
& FDR$\downarrow$ & Power$\uparrow$
& FDR$\downarrow$ & Power$\uparrow$
& FDR$\downarrow$ & Power$\uparrow$
& FDR$\downarrow$ & Power$\uparrow$ \\
\midrule
\multicolumn{9}{l}{\textit{Random}} \\
\midrule
Regular
& 0.0972 ($\pm$ 0.0110) & 70.46 ($\pm$ 1.54)
& 0.0965 ($\pm$ 0.0102) & 71.43 ($\pm$ 1.29)
& 0.0980 ($\pm$ 0.0118) & 73.28 ($\pm$ 1.25)
& 0.0972 ($\pm$ 0.0110) & 71.72 ($\pm$ 1.36)
\\
VCD (\citeyear{leng2024mitigating})
& 0.0913 ($\pm$ 0.0115) & 75.11 ($\pm$ 1.62)
& 0.0920 ($\pm$ 0.0110) & 74.48 ($\pm$ 1.15)
& 0.0916 ($\pm$ 0.0114) & 75.00 ($\pm$ 1.01)
& 0.0916 ($\pm$ 0.0113) & 74.86 ($\pm$ 1.26)\\
MARINE (\citeyear{zhao2024mitigating})
& 0.0876 ($\pm$ 0.0136) & 77.89 ($\pm$ 1.25)
& 0.0869 ($\pm$ 0.0133) & 78.03 ($\pm$ 1.65)
& 0.0866 ($\pm$ 0.0151) & 78.58 ($\pm$ 1.85)
& 0.0870 ($\pm$ 0.0140) & 78.17 ($\pm$ 1.58)\\
AGLA (\citeyear{an2025mitigating})
& \uline{0.0801 ($\pm$ 0.0197)} & \uline{82.84 ($\pm$ 1.54)}
& \uline{0.0796 ($\pm$ 0.0170)} & \uline{83.15 ($\pm$ 2.07)}
& \uline{0.0787 ($\pm$ 0.0200)} & \uline{83.69 ($\pm$ 1.96)}
& \uline{0.0795 ($\pm$ 0.0189)} & \uline{83.23 ($\pm$ 1.86)} \\
\rowcolor{blue!8}
$\method$ (Ours)
& \textbf{0.0727 ($\pm$ 0.0201)} & \textbf{89.28 ($\pm$ 1.63)}
& \textbf{0.0719 ($\pm$ 0.0197)} & \textbf{90.79 ($\pm$ 2.15)}
& \textbf{0.0726 ($\pm$ 0.0163)} & \textbf{90.76 ($\pm$ 2.05)}
& \textbf{0.0724 ($\pm$ 0.0187)} & \textbf{90.28 ($\pm$ 1.94)} \\
\bottomrule
\end{tabular}
}
\end{table*}

\begin{table}[t]
\centering
\caption{Evaluation with POPE score across multiple LVLM architectures on \textbf{MSCOCO} dataset comparing our method with several baselines. \textbf{Bold} indicates the best result and \uline{underline} indicates the second-best result in each column.}
\label{tab:pope_coco1_random}

\setlength{\tabcolsep}{4pt}
\renewcommand{\arraystretch}{1}
\resizebox{\textwidth}{!}{
\begin{tabular}{lcccccccc}
\toprule
\multirow{2}{*}{Decoding}
& \multicolumn{4}{c}{LLaVA-v1.5}
& \multicolumn{4}{c}{Qwen-VL} \\
\cmidrule(lr){2-5}\cmidrule(lr){6-9}
& Accuracy$\uparrow$ & Precision$\uparrow$ & Recall$\uparrow$ & F1 Score$\uparrow$
& Accuracy$\uparrow$ & Precision$\uparrow$ & Recall$\uparrow$ & F1 Score$\uparrow$ \\
\midrule

\multicolumn{9}{l}{\textit{Random}} \\
\midrule
Regular
& 83.29 ($\pm$ 0.35) & \uline{92.13 ($\pm$ 0.54)} & 72.80 ($\pm$ 0.57) & 81.33 ($\pm$ 0.41)
& 84.73 ($\pm$ 0.36) & 95.61 ($\pm$ 0.45) & 72.81 ($\pm$ 0.38) & 82.67 ($\pm$ 0.41) \\
VCD~(\citeyear{leng2024mitigating})
& 87.73 ($\pm$ 0.40) & 91.42 ($\pm$ 0.55) & \uline{83.28 ($\pm$ 0.42)} & 87.16 ($\pm$ 0.41)
& \uline{88.63 ($\pm$ 0.10)} & 94.64 ($\pm$ 0.25) & 81.91 ($\pm$ 0.19) & \uline{87.81 ($\pm$ 0.11)} \\
MARINE~(\citeyear{zhao2024mitigating})
& 85.01 ($\pm$ 0.24) & 88.27 ($\pm$ 0.83) & 80.73 ($\pm$ 0.12) & 84.33 ($\pm$ 0.31)
& 82.07 ($\pm$ 0.14) & 89.27 ($\pm$ 0.13) & \textbf{89.33 ($\pm$ 0.24)} & 85.83 ($\pm$ 0.87) \\
AGLA~(\citeyear{an2025mitigating})
& \uline{88.54 ($\pm$ 0.64)} & \textbf{94.41 ($\pm$ 0.50)} & 82.08 ($\pm$ 0.47) & \uline{87.71 ($\pm$ 0.51)}
& 84.60 ($\pm$ 0.76) & \textbf{98.23 ($\pm$ 0.29)} & 70.47 ($\pm$ 0.41) & 82.07 ($\pm$ 0.11) \\
\rowcolor{blue!8} $\method$ (Ours)
& \textbf{90.33 ($\pm$ 0.54)} & 90.77 ($\pm$ 0.76) & \textbf{89.80 ($\pm$ 0.78)} & \textbf{90.28 ($\pm$ 0.57)}
& \textbf{91.33 ($\pm$ 0.51)} & \uline{96.20 ($\pm$ 0.52)} & \uline{86.07 ($\pm$ 0.90)} & \textbf{90.85 ($\pm$ 0.57)} \\
\bottomrule
\end{tabular}
}

\setlength{\tabcolsep}{4pt}
\renewcommand{\arraystretch}{1}
\resizebox{\textwidth}{!}{
\begin{tabular}{lcccccccc}
\toprule
\multirow{2}{*}{Decoding}
& \multicolumn{4}{c}{InstructBLIP}
& \multicolumn{4}{c}{LLaVA-OneVision-7B} \\
\cmidrule(lr){2-5}\cmidrule(lr){6-9}
& Accuracy$\uparrow$ & Precision$\uparrow$ & Recall$\uparrow$ & F1 Score$\uparrow$
& Accuracy$\uparrow$ & Precision$\uparrow$ & Recall$\uparrow$ & F1 Score$\uparrow$ \\
\midrule

\multicolumn{9}{l}{\textit{Random}} \\
\midrule
Regular
& 80.71 ($\pm$ 0.73) & 81.67 ($\pm$ 0.67) & 79.19 ($\pm$ 1.14) & 80.41 ($\pm$ 0.80)
& 85.87 ($\pm$ 1.77) & 83.41 ($\pm$ 2.13) & 88.08 ($\pm$ 1.47) & 85.72 ($\pm$ 1.66) \\

VCD~(\citeyear{leng2024mitigating})
& 84.53 ($\pm$ 0.38) & 88.55 ($\pm$ 0.54) & 79.32 ($\pm$ 0.44) & 83.68 ($\pm$ 0.40)
& 86.33 ($\pm$ 1.23) & 89.01 ($\pm$ 1.42) & 87.33 ($\pm$ 1.78) & 88.86 ($\pm$ 1.43) \\

MARINE~(\citeyear{zhao2024mitigating})
& 86.72 ($\pm$ 0.11) & 87.97 ($\pm$ 0.53) & 80.03 ($\pm$ 0.19) & 87.33 ($\pm$ 0.91)
& 87.21 ($\pm$ 1.35) & 90.33 ($\pm$ 1.79) & 88.11 ($\pm$ 1.93) & 89.09 ($\pm$ 1.09) \\

AGLA~(\citeyear{an2025mitigating})
& \uline{87.30 ($\pm$ 0.32)} & \uline{88.83 ($\pm$ 0.41)} & \uline{82.08 ($\pm$ 0.47)} & \uline{87.71 ($\pm$ 0.51)}
& \uline{88.15 ($\pm$ 1.65)} & \uline{93.26 ($\pm$ 1.88)} & \textbf{90.08 ($\pm$ 1.47)} & \uline{89.97 ($\pm$ 1.21)} \\

\rowcolor{blue!8} $\method$ (Ours)
& \textbf{90.13 ($\pm$ 0.54)} & \textbf{92.63 ($\pm$ 0.47)} & \textbf{87.20 ($\pm$ 0.61)} & \textbf{89.84 ($\pm$ 0.38)}
& \textbf{92.17 ($\pm$ 0.89)} & \textbf{98.25 ($\pm$ 1.46)} & \uline{89.87 ($\pm$ 1.61)} & \textbf{91.64 ($\pm$ 1.83)} \\
\bottomrule
\end{tabular}
}


\setlength{\tabcolsep}{4pt}
\renewcommand{\arraystretch}{1}
\resizebox{\textwidth}{!}{
\begin{tabular}{lcccccccc}
\toprule
\multirow{2}{*}{Decoding}
& \multicolumn{4}{c}{Qwen2.5-VL-7B}
& \multicolumn{4}{c}{InternVL3-8B} \\
\cmidrule(lr){2-5}\cmidrule(lr){6-9}
& Accuracy$\uparrow$ & Precision$\uparrow$ & Recall$\uparrow$ & F1 Score$\uparrow$
& Accuracy$\uparrow$ & Precision$\uparrow$ & Recall$\uparrow$ & F1 Score$\uparrow$ \\
\midrule

\multicolumn{9}{l}{\textit{Random}} \\
\midrule
Regular
& 86.20 ($\pm$ 1.30) & 88.68 ($\pm$ 1.54) & 87.62 ($\pm$ 1.24) & 86.98 ($\pm$ 1.28)
& 87.15 ($\pm$ 1.88) & 86.87 ($\pm$ 1.67) & 87.24 ($\pm$ 1.76) & 86.26 ($\pm$ 1.80) \\

VCD~(\citeyear{leng2024mitigating})
& 87.35 ($\pm$ 1.83) & 89.55 ($\pm$ 1.45) & 89.23 ($\pm$ 1.35) & 90.06 ($\pm$ 2.03)
& 88.15 ($\pm$ 1.53) & 89.01 ($\pm$ 1.17) & 88.24 ($\pm$ 1.44) & 88.05 ($\pm$ 1.45) \\

MARINE~(\citeyear{zhao2024mitigating})
& 89.72 ($\pm$ 1.11) & 91.79 ($\pm$ 1.35) & 91.31 ($\pm$ 1.91) & 90.33 ($\pm$ 2.09)
& 89.23 ($\pm$ 1.35) & 90.24 ($\pm$ 1.66) & 88.52 ($\pm$ 1.15) & 89.15 ($\pm$ 1.19) \\

AGLA~(\citeyear{an2025mitigating})
& \uline{90.01 ($\pm$ 1.23)} & \uline{93.38 ($\pm$ 1.14)} & \uline{92.08 ($\pm$ 1.74)} & \uline{91.17 ($\pm$ 1.51)}
& \uline{92.53 ($\pm$ 1.42)} & \uline{93.86 ($\pm$ 1.63)} & \textbf{92.52 ($\pm$ 1.11)} & \uline{93.54 ($\pm$ 1.77)} \\

\rowcolor{blue!8} $\method$ (Ours)
& \textbf{90.63 ($\pm$ 1.54)} & \textbf{99.71 ($\pm$ 1.74)} & \textbf{96.47 ($\pm$ 1.68)} & \textbf{91.89 ($\pm$ 1.83)}
& \textbf{93.55 ($\pm$ 1.52)} & \textbf{95.83 ($\pm$ 1.33)} & \uline{91.50 ($\pm$ 1.27)} & \textbf{95.35 ($\pm$ 1.87)} \\
\bottomrule
\end{tabular}
}
\end{table}

\begin{table*}[htbp]
\centering
\caption{Evaluation with POPE score across multiple LVLM architectures on \textbf{A-OKVQA} dataset comparing our method with several baselines with 3000 random replications. \textbf{Bold} indicates the best result and \uline{underline} indicates the second-best result in each column.}
\label{tab:pope_vaq}
\setlength{\tabcolsep}{4pt}
\renewcommand{\arraystretch}{1}
\resizebox{\textwidth}{!}{%
\begin{tabular}{lcccccccc}
\toprule
Decoding & \multicolumn{4}{c}{LLaVA-v1.5} & \multicolumn{4}{c}{Qwen-VL} \\
\cmidrule(lr){2-5}\cmidrule(lr){6-9}
& Accuracy$\uparrow$ & Precision$\uparrow$ & Recall$\uparrow$ & F1 Score$\uparrow$
& Accuracy$\uparrow$ & Precision$\uparrow$ & Recall$\uparrow$ & F1 Score$\uparrow$ \\
\midrule

\textit{Random} \\
\midrule
Regular
& 83.45 ($\pm$ 0.48) & 87.24 ($\pm$ 0.68) & 78.36 ($\pm$ 0.54) & 82.56 ($\pm$ 0.50)
& 86.67 ($\pm$ 0.48) & \uline{93.16 ($\pm$ 0.55)} & 79.16 ($\pm$ 0.59) & 85.59 ($\pm$ 0.53) \\

VCD~(\citeyear{leng2024mitigating})
& 86.15 ($\pm$ 0.23) & 85.18 ($\pm$ 0.34) & \uline{87.53 ($\pm$ 0.14)} & 86.34 ($\pm$ 0.21)
& \uline{89.22 ($\pm$ 0.14)} & 90.77 ($\pm$ 0.04) & \uline{87.32 ($\pm$ 0.34)} & \uline{89.01 ($\pm$ 0.16)} \\

MARINE~(\citeyear{zhao2024mitigating})
& 86.72 ($\pm$ 0.14) & \uline{87.71 ($\pm$ 0.53)} & \uline{87.53 ($\pm$ 0.34)} & 86.34 ($\pm$ 0.56)
& 89.17 ($\pm$ 0.45) & 89.87 ($\pm$ 0.42) & \textbf{88.13 ($\pm$ 0.54)} & 88.93 ($\pm$ 0.18) \\

AGLA~(\citeyear{an2025mitigating})
& \textbf{89.28 ($\pm$ 0.33)} & \textbf{93.18 ($\pm$ 0.43)} & 84.76 ($\pm$ 0.42) & \uline{88.77 ($\pm$ 0.56)}
& 86.77 ($\pm$ 0.43) & \textbf{95.02 ($\pm$ 0.11)} & 77.60 ($\pm$ 0.72) & 85.43 ($\pm$ 0.98) \\

\rowcolor{blue!8} $\method$ (Ours)
& \uline{88.90 ($\pm$ 0.52)} & 83.52 ($\pm$ 0.49) & \textbf{96.93 ($\pm$ 0.64)} & \textbf{89.73 ($\pm$ 0.41)}
& \textbf{89.53 ($\pm$ 0.14)} & 90.92 ($\pm$ 0.04) & 86.42 ($\pm$ 0.58) & \textbf{89.10 ($\pm$ 0.37)} \\
\bottomrule
\end{tabular}
}
\setlength{\tabcolsep}{4pt}
\renewcommand{\arraystretch}{1}
\resizebox{\textwidth}{!}{%
\begin{tabular}{lcccccccc}
\toprule
Decoding
& \multicolumn{4}{c}{InstructBLIP}
& \multicolumn{4}{c}{LLaVA-OneVision-7B} \\
\cmidrule(lr){2-5}\cmidrule(lr){6-9}
& Accuracy$\uparrow$ & Precision$\uparrow$ & Recall$\uparrow$ & F1 Score$\uparrow$
& Accuracy$\uparrow$ & Precision$\uparrow$ & Recall$\uparrow$ & F1 Score$\uparrow$ \\
\midrule
\textit{Random} \\
\midrule
Regular
& 80.71 ($\pm$ 0.73) & 81.67 ($\pm$ 0.67) & 79.19 ($\pm$ 1.14) & 80.41 ($\pm$ 0.80)
& 85.87 ($\pm$ 1.77) & 83.41 ($\pm$ 2.13) & 88.08 ($\pm$ 1.47) & 85.72 ($\pm$ 1.66) \\

VCD~(\citeyear{leng2024mitigating})
& 84.53 ($\pm$ 0.38) & 88.55 ($\pm$ 0.54) & 79.32 ($\pm$ 0.44) & 83.68 ($\pm$ 0.40)
& 86.33 ($\pm$ 1.23) & 89.01 ($\pm$ 1.42) & 87.33 ($\pm$ 1.78) & 88.86 ($\pm$ 1.43) \\

MARINE~(\citeyear{zhao2024mitigating})
& 86.72 ($\pm$ 0.11) & 87.97 ($\pm$ 0.53) & 80.03 ($\pm$ 0.19) & 87.33 ($\pm$ 0.91)
& 87.21 ($\pm$ 1.35) & 90.33 ($\pm$ 1.79) & 88.11 ($\pm$ 1.93) & 89.09 ($\pm$ 1.09) \\

AGLA~(\citeyear{an2025mitigating})
& \uline{87.30 ($\pm$ 0.32)} & \uline{88.83 ($\pm$ 0.41)} & \uline{82.08 ($\pm$ 0.47)} & \uline{87.71 ($\pm$ 0.51)}
& \uline{88.15 ($\pm$ 1.65)} & \uline{93.26 ($\pm$ 1.88)} & \textbf{90.08 ($\pm$ 1.47)} & \uline{89.97 ($\pm$ 1.21)} \\

\rowcolor{blue!8} $\method$ (Ours)
& \textbf{90.13 ($\pm$ 0.54)} & \textbf{92.63 ($\pm$ 0.47)} & \textbf{87.20 ($\pm$ 0.61)} & \textbf{89.84 ($\pm$ 0.38)}
& \textbf{92.17 ($\pm$ 0.89)} & \textbf{98.25 ($\pm$ 1.46)} & \uline{89.87 ($\pm$ 1.61)} & \textbf{91.64 ($\pm$ 1.83)} \\
\bottomrule
\end{tabular}
}
\setlength{\tabcolsep}{4pt}
\renewcommand{\arraystretch}{1}
\resizebox{\textwidth}{!}{%
\begin{tabular}{lcccccccc}
\toprule
Decoding
& \multicolumn{4}{c}{Qwen2.5-VL-7B}
& \multicolumn{4}{c}{InternVL3-8B} \\
\cmidrule(lr){2-5}\cmidrule(lr){6-9}
& Accuracy$\uparrow$ & Precision$\uparrow$ & Recall$\uparrow$ & F1 Score$\uparrow$
& Accuracy$\uparrow$ & Precision$\uparrow$ & Recall$\uparrow$ & F1 Score$\uparrow$ \\
\midrule

\textit{Random} \\
\midrule
Regular
& 86.20 ($\pm$ 1.30) & 88.68 ($\pm$ 1.54) & 87.62 ($\pm$ 1.24) & 86.98 ($\pm$ 1.28)
& 87.15 ($\pm$ 1.88) & 86.87 ($\pm$ 1.67) & 87.24 ($\pm$ 1.76) & 86.26 ($\pm$ 1.80) \\

VCD~(\citeyear{leng2024mitigating})
& 87.35 ($\pm$ 1.83) & 89.55 ($\pm$ 1.45) & 89.23 ($\pm$ 1.35) & 90.06 ($\pm$ 2.03)
& 88.15 ($\pm$ 1.53) & 89.01 ($\pm$ 1.17) & 88.24 ($\pm$ 1.44) & 88.05 ($\pm$ 1.45) \\

MARINE~(\citeyear{zhao2024mitigating})
& 89.72 ($\pm$ 1.11) & 91.79 ($\pm$ 1.35) & 91.31 ($\pm$ 1.91) & 90.33 ($\pm$ 2.09)
& 89.23 ($\pm$ 1.35) & 90.24 ($\pm$ 1.66) & 88.52 ($\pm$ 1.15) & 89.15 ($\pm$ 1.19) \\

AGLA~(\citeyear{an2025mitigating})
& \uline{90.01 ($\pm$ 1.23)} & \uline{93.38 ($\pm$ 1.14)} & \uline{92.08 ($\pm$ 1.74)} & \uline{91.17 ($\pm$ 1.51)}
& \uline{92.53 ($\pm$ 1.42)} & \uline{93.86 ($\pm$ 1.63)} & \textbf{92.52 ($\pm$ 1.11)} & \uline{93.54 ($\pm$ 1.77)} \\

\rowcolor{blue!8} $\method$ (Ours)
& \textbf{90.63 ($\pm$ 1.54)} & \textbf{99.71 ($\pm$ 1.74)} & \textbf{96.47 ($\pm$ 1.68)} & \textbf{91.89 ($\pm$ 1.83)}
& \textbf{93.55 ($\pm$ 1.52)} & \textbf{95.83 ($\pm$ 1.33)} & \uline{91.50 ($\pm$ 1.27)} & \textbf{95.35 ($\pm$ 1.87)} \\
\bottomrule
\end{tabular}
}
\end{table*}

\begin{table}[ht]
\centering
\caption{Evaluation with POPE score across multiple LVLM architectures on \textbf{GQA} dataset, comparing our method with several baselines with 3000 random replications. \textbf{Bold} indicates the best result and \uline{underline} indicates the second-best result in each column.
}
\label{tab:pope_gqa}
\setlength{\tabcolsep}{4pt}
\renewcommand{\arraystretch}{1}
\resizebox{\textwidth}{!}{%
\begin{tabular}{lcccccccc}
\toprule
Decoding & \multicolumn{4}{c}{LLaVA-v1.5} & \multicolumn{4}{c}{Qwen-VL} \\
\cmidrule(lr){2-5}\cmidrule(lr){6-9}
& Accuracy$\uparrow$ & Precision$\uparrow$ & Recall$\uparrow$ & F1 Score$\uparrow$
& Accuracy$\uparrow$ & Precision$\uparrow$ & Recall$\uparrow$ & F1 Score$\uparrow$ \\
\midrule

\textit{Random} \\
\midrule
Regular
& 83.73 ($\pm$ 0.27) & \uline{87.16 ($\pm$ 0.39)} & 79.12 ($\pm$ 0.35) & 82.95 ($\pm$ 0.28)
& 80.97 ($\pm$ 0.32) & 88.07 ($\pm$ 0.34) & 71.64 ($\pm$ 0.57) & 79.01 ($\pm$ 0.40) \\

VCD~(\citeyear{leng2024mitigating})
& \uline{86.65 ($\pm$ 0.45)} & 84.58 ($\pm$ 0.59) & \uline{89.24 ($\pm$ 0.34)} & 86.99 ($\pm$ 0.41)
& \uline{85.59 ($\pm$ 0.38)} & 86.88 ($\pm$ 0.44) & 83.84 ($\pm$ 0.36) & 85.33 ($\pm$ 0.38) \\

MARINE~(\citeyear{zhao2024mitigating})
& 86.33 ($\pm$ 0.14) & 85.02 ($\pm$ 0.63) & 88.23 ($\pm$ 0.15) & \uline{87.24 ($\pm$ 0.11)}
& 85.54 ($\pm$ 0.52) & 87.35 ($\pm$ 0.25) & \uline{87.26 ($\pm$ 0.22)} & \uline{86.63 ($\pm$ 0.22)} \\

AGLA~(\citeyear{an2025mitigating})
& 86.46 ($\pm$ 0.34) & 85.84 ($\pm$ 0.25) & 87.31 ($\pm$ 0.15) & 86.57 ($\pm$ 0.15)
& 83.90 ($\pm$ 0.43) & \textbf{93.05 ($\pm$ 0.09)} & 73.26 ($\pm$ 0.82) & 81.98 ($\pm$ 0.13) \\

\rowcolor{blue!8} $\method$ (Ours)
& \textbf{88.02 ($\pm$ 0.24)} & \textbf{87.25 ($\pm$ 0.26)} & \textbf{90.02 ($\pm$ 0.13)} & \textbf{89.11 ($\pm$ 0.42)}
& \textbf{89.13 ($\pm$ 0.34)} & \uline{89.96 ($\pm$ 0.21)} & \textbf{90.04 ($\pm$ 0.23)} & \textbf{89.52 ($\pm$ 0.63)} \\
\bottomrule
\end{tabular}
}
\setlength{\tabcolsep}{4pt}
\renewcommand{\arraystretch}{1}
\resizebox{\textwidth}{!}{%
\begin{tabular}{lcccccccc}
\toprule
Decoding
& \multicolumn{4}{c}{InstructBLIP}
& \multicolumn{4}{c}{LLaVA-OneVision-7B} \\
\cmidrule(lr){2-5}\cmidrule(lr){6-9}
& Accuracy$\uparrow$ & Precision$\uparrow$ & Recall$\uparrow$ & F1 Score$\uparrow$
& Accuracy$\uparrow$ & Precision$\uparrow$ & Recall$\uparrow$ & F1 Score$\uparrow$ \\
\midrule
\textit{Random} \\
\midrule
Regular
& 79.65 ($\pm$ 0.24) & 77.14 ($\pm$ 0.43) & 84.26 ($\pm$ 0.36) & 80.56 ($\pm$ 0.18)
& 85.98 ($\pm$ 0.14) & 86.53 ($\pm$ 0.37) & 84.16 ($\pm$ 0.42) & 85.27 ($\pm$ 0.15) \\

VCD~(\citeyear{leng2024mitigating})
& 83.69 ($\pm$ 0.11) & 81.84 ($\pm$ 0.42) & 86.61 ($\pm$ 0.48) & 84.16 ($\pm$ 0.01)
& 86.16 ($\pm$ 0.13) & 85.26 ($\pm$ 0.44) & 84.42 ($\pm$ 0.17) & 85.37 ($\pm$ 0.16) \\

MARINE~(\citeyear{zhao2024mitigating})
& 85.25 ($\pm$ 0.21) & 83.54 ($\pm$ 0.14) & 86.92 ($\pm$ 0.09) & 85.25 ($\pm$ 0.53)
& 88.33 ($\pm$ 0.53) & 86.26 ($\pm$ 0.14) & 87.33 ($\pm$ 0.25) & 86.65 ($\pm$ 0.15) \\

AGLA~(\citeyear{an2025mitigating})
& \uline{86.46 ($\pm$ 0.35)} & \uline{85.84 ($\pm$ 0.22)} & \uline{87.31 ($\pm$ 0.54)} & \uline{86.57 ($\pm$ 0.25)}
& \uline{89.26 ($\pm$ 0.42)} & \uline{89.72 ($\pm$ 0.17)} & \uline{88.01 ($\pm$ 0.35)} & \uline{88.53 ($\pm$ 0.35)} \\

\rowcolor{blue!8} $\method$ (Ours)
& \textbf{90.25 ($\pm$ 0.24)} & \textbf{89.63 ($\pm$ 0.54)} & \textbf{90.21 ($\pm$ 0.20)} & \textbf{90.01 ($\pm$ 0.25)}
& \textbf{90.18 ($\pm$ 0.16)} & \textbf{91.36 ($\pm$ 0.35)} & \textbf{90.01 ($\pm$ 0.35)} & \textbf{90.33 ($\pm$ 0.31)} \\
\bottomrule
\end{tabular}
}
\setlength{\tabcolsep}{4pt}
\renewcommand{\arraystretch}{1}
\resizebox{\textwidth}{!}{%
\begin{tabular}{lcccccccc}
\toprule
Decoding
& \multicolumn{4}{c}{Qwen2.5-VL-7B}
& \multicolumn{4}{c}{InternVL3-8B} \\
\cmidrule(lr){2-5}\cmidrule(lr){6-9}
& Accuracy$\uparrow$ & Precision$\uparrow$ & Recall$\uparrow$ & F1 Score$\uparrow$
& Accuracy$\uparrow$ & Precision$\uparrow$ & Recall$\uparrow$ & F1 Score$\uparrow$ \\
\midrule

\textit{Random} \\
\midrule
Regular
& 86.19 ($\pm$ 0.26) & 85.05 ($\pm$ 0.24) & 85.37 ($\pm$ 0.26) & 86.00 ($\pm$ 0.26)
& 85.74 ($\pm$ 0.15) & 86.33 ($\pm$ 0.33) & 85.56 ($\pm$ 0.56) & 85.32 ($\pm$ 0.27) \\

VCD~(\citeyear{leng2024mitigating})
& 87.14 ($\pm$ 0.74) & 87.58 ($\pm$ 0.11) & 86.74 ($\pm$ 0.34) & 85.36 ($\pm$ 0.15)
& 85.14 ($\pm$ 0.26) & 86.52 ($\pm$ 0.16) & 86.16 ($\pm$ 0.11) & 85.21 ($\pm$ 0.26) \\

MARINE~(\citeyear{zhao2024mitigating})
& 88.01 ($\pm$ 0.15) & 89.79 ($\pm$ 0.25) & 87.17 ($\pm$ 0.13) & 88.10 ($\pm$ 0.15)
& 87.33 ($\pm$ 0.14) & 89.53 ($\pm$ 0.31) & 86.68 ($\pm$ 0.13) & 87.42 ($\pm$ 0.25) \\

AGLA~(\citeyear{an2025mitigating})
& \uline{90.22 ($\pm$ 0.14)} & \uline{90.66 ($\pm$ 0.32)} & \uline{89.71 ($\pm$ 0.35)} & \uline{90.20 ($\pm$ 0.45)}
& \uline{90.04 ($\pm$ 0.24)} & \uline{89.99 ($\pm$ 0.32)} & \uline{90.79 ($\pm$ 0.24)} & \uline{89.22 ($\pm$ 0.42)} \\

\rowcolor{blue!8} $\method$ (Ours)
& \textbf{91.03 ($\pm$ 0.33)} & \textbf{91.39 ($\pm$ 0.31)} & \textbf{90.40 ($\pm$ 0.15)} & \textbf{90.65 ($\pm$ 0.15)}
& \textbf{92.15 ($\pm$ 0.33)} & \textbf{93.36 ($\pm$ 0.14)} & \textbf{91.36 ($\pm$ 0.11)} & \textbf{91.16 ($\pm$ 0.30)} \\
\bottomrule
\end{tabular}
}
\end{table}

\section{Experiments for Ablation Study}
\subsection{Effect of FDR Threshold}
We also examine the role of the false discovery rate (FDR) control procedure on the MSCOCO dataset across three models. In this ablation study, we remove the adaptive threshold selection based on the estimated FDR and instead apply fixed thresholds chosen on a validation set or select a fixed proportion of tokens. This variant preserves the mirror statistic but disables explicit error control. The results show that without FDR-based thresholding, the empirical FDR varies substantially across experimental settings on MSCOCO, often exceeding the target level. In contrast, the full method consistently maintains empirical FDR close to the desired rate while achieving comparable or better hallucination reduction. These findings indicate that the performance gains of our approach are not solely attributable to the mirror statistic but critically rely on the FDR control procedure, which provides explicit and stable error control. The corresponding results are reported in Figure~\ref{fig:q_level}.
\begin{figure}
    \centering
    \includegraphics[width=\linewidth]{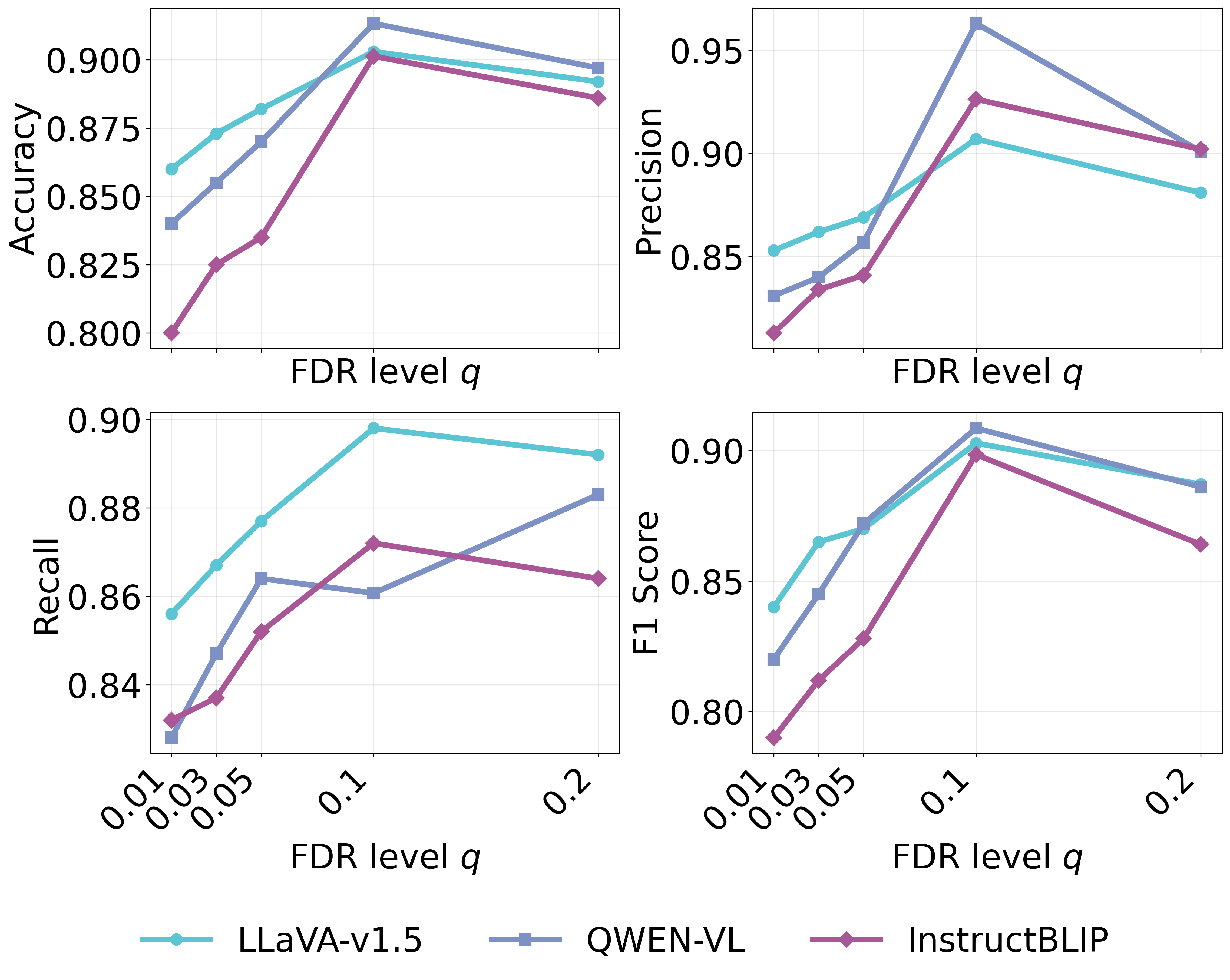}
    \caption{Ablation study on the effect of FDR target level ($q$) on the performance of LLaVA-v1.5, Qwen-VL, InstructBLIP using POPE metrics with $q = \{0.01, 0.03, 0.05, 0.1, 0.2\}$.}
    \label{fig:q_level}
\end{figure}

\subsection{Calibration of the Perturbation Scale $\tau_v$}
\label{app:tau}

As described in the main text, the perturbation scale $\tau_v$ is calibrated using a lightweight sensitivity analysis rather than fine-grained hyperparameter optimization. Specifically, $\tau_v$ is selected based on a coarse sweep over a
representative range of values on a small subset disjoint from the
reported evaluation images, with the target false discovery rate (FDR)
level fixed at $q=0.1$. The goal of this calibration is to identify a stable operating regime that maintains empirical FDR close to the target level while yielding robust power, rather than to optimize any task-specific metric.

Once selected, the perturbation scale $\tau_v$ is kept fixed across all datasets and experiments for each backbone, and no further tuning is performed on the test benchmarks. Figures~\ref{fig:ablation_fdr_tau} and~\ref{fig:ablation_pope_tau} report the corresponding sensitivity analysis. The results show that $\method$ exhibits smooth and consistent behavior across a wide range of $\tau_v$ values, indicating that the method is not sensitive to precise tuning of the perturbation scale and mitigating the risk of overfitting to any specific dataset or evaluation protocol.

\begin{figure}[ht]
    \centering
    \includegraphics[width=\linewidth]{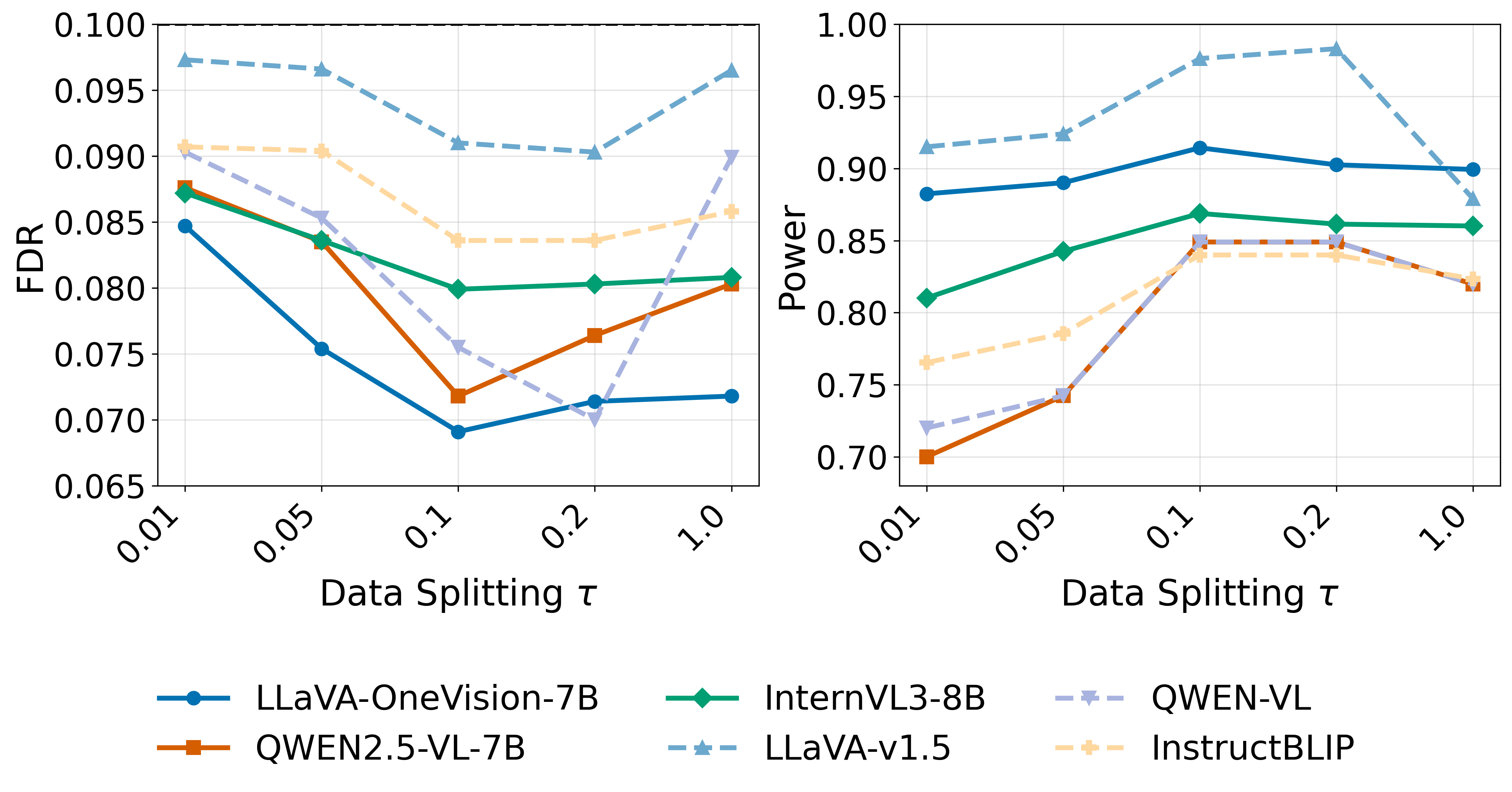}
    \caption{Sensitivity analysis of FDR and power with respect to the perturbation scale $\tau_v$. We report empirical FDR and power under different values of $\tau_v \in \{0.01, 0.05, 0.10, 0.20, 1.00\}$ while fixing the target FDR level at $q = 0.1$. The results show that $\method$ achieves a favorable trade-off at intermediate perturbation scales and remains robust across a wide range of $\tau_v$ values.}
    \label{fig:ablation_fdr_tau}
\end{figure}

\begin{figure}[ht]
    \centering
    \includegraphics[width=\linewidth]{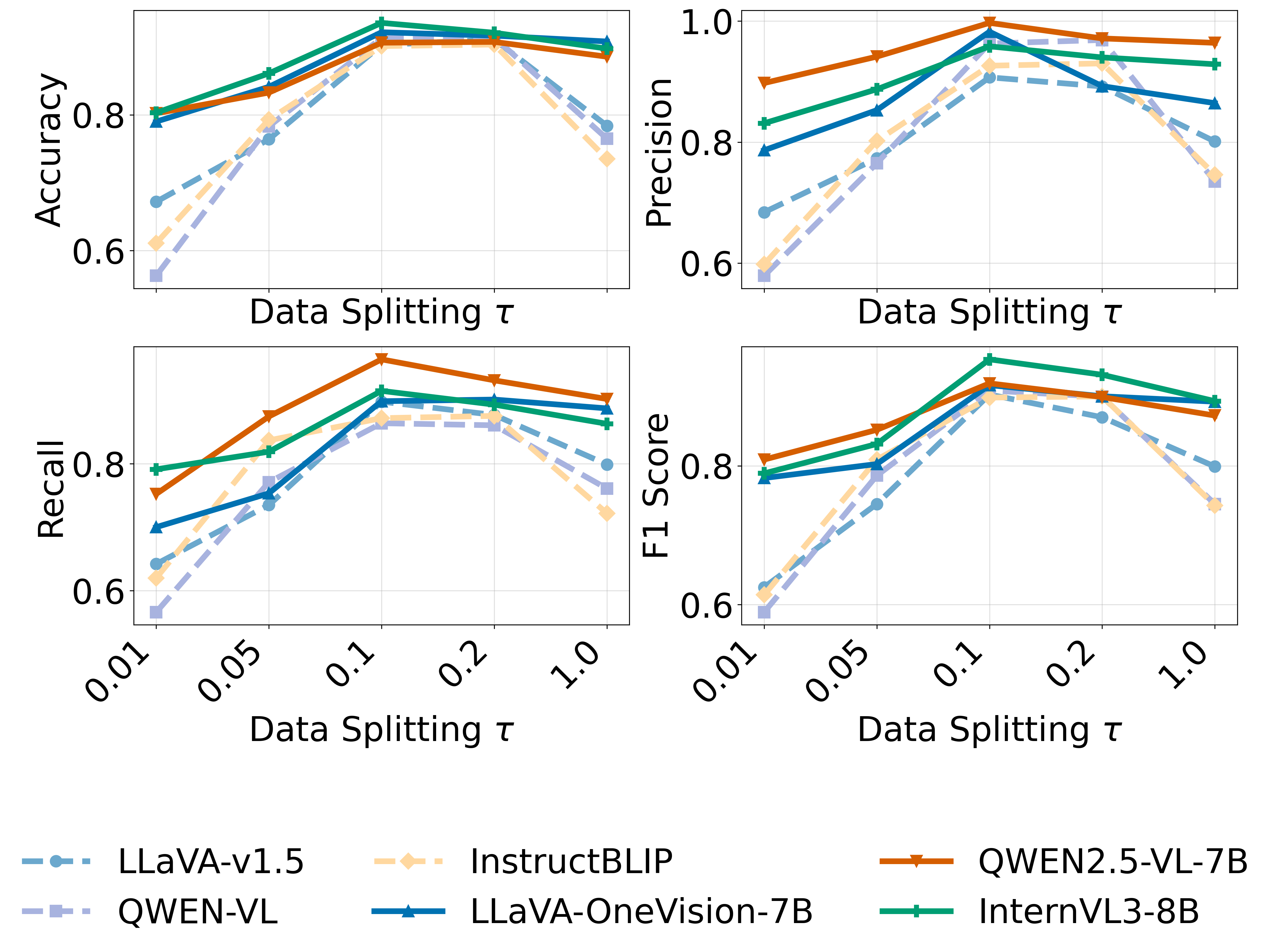}
    \caption{Sensitivity analysis of POPE accuracy and F1 with respect to the perturbation scale $\tau_v$. Performance is evaluated under different values of $\tau_v \in \{0.01, 0.05, 0.10, 0.20, 1.00\}$ while fixing the target FDR level at $q = 0.1$. The results indicate that $\method$ is not sensitive to precise tuning of $\tau_v$ and achieves stable performance across a broad range of perturbation scales.}
    \label{fig:ablation_pope_tau}
\end{figure}

\section{Boundary and Failure Cases.}
Figure~\ref{fig:failure_cases} presents representative failure
cases of CORAL, including false negatives for present objects,
confusion between visually similar categories such as a
streetlamp and a traffic light, difficulties recognizing
small or partially occluded objects such as a knife, and
confusion between related activities.
These examples illustrate that errors can remain when visual
evidence is weak or semantic distinctions are fine-grained.

One possible explanation is that perturbation responses for
grounded and non-grounded concepts are insufficiently
separated, limiting the discriminative ability of the mirror
statistic. However, the qualitative examples alone do not
establish this mechanism or imply that the corresponding
statistics lie near the selection threshold.
These cases complement the quantitative evaluation
by highlighting the limitations of CORAL. The hallucination
mitigation does not eliminate all unsupported predictions
and may also reject genuinely present objects.

\begin{figure}[ht]
    \centering
    \includegraphics[width=\linewidth]{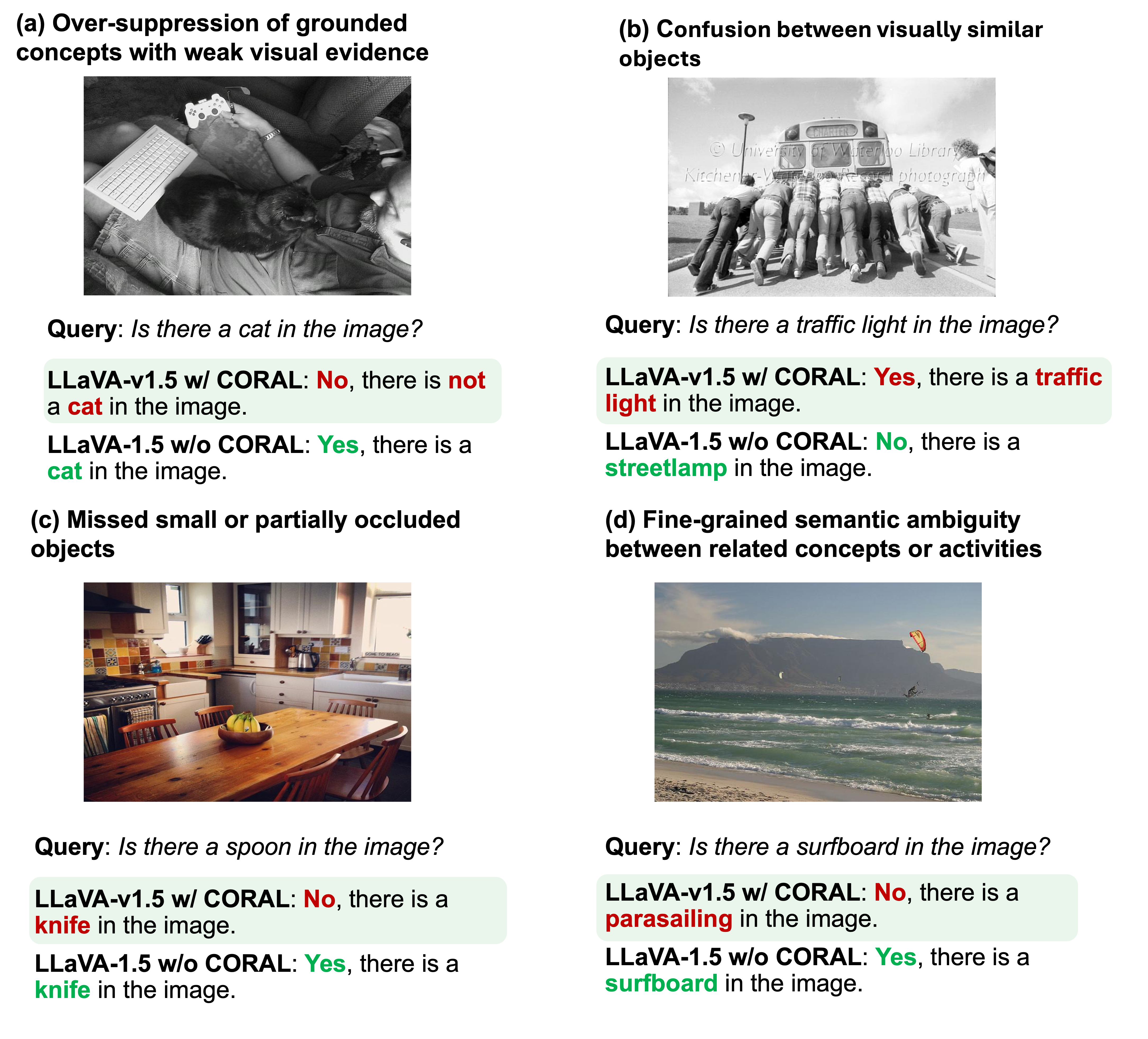}
    \caption{Representative failure cases of $\method$ on LLaVA-v1.5. Each example compares responses with and without $\method$ to the same image and query. The cases illustrate remaining errors in object recognition and activity description, including failure to recognize a present cat, an unsupported traffic-light prediction, an inconsistent response concerning a knife, and confusion involving a surfboard and parasailing. Red and green text highlight incorrect and correct response components, respectively.}
    \label{fig:failure_cases}
\end{figure}


\end{document}